\documentclass{article} % For LaTeX2e
\usepackage{iclr2020_conference,times}

\usepackage{amsmath,amsfonts,bm}

\def\eqref#1{equation~\ref{#1}}
\def\1{\bm{1}}

\DeclareMathAlphabet{\mathsfit}{\encodingdefault}{\sfdefault}{m}{sl}
\SetMathAlphabet{\mathsfit}{bold}{\encodingdefault}{\sfdefault}{bx}{n}

\newcommand{\Var}{\mathrm{Var}}

\usepackage{hyperref}
\usepackage{url}

\usepackage[utf8]{inputenc} % allow utf-8 input
\usepackage[T1]{fontenc}    % use 8-bit T1 fonts
\usepackage{booktabs}       % professional-quality tables
\usepackage{amsfonts}       % blackboard math symbols
\usepackage{nicefrac}       % compact symbols for 1/2, etc.
\usepackage{microtype}      % microtypography

\usepackage{mysymbols}
\usepackage{attrib}
\usepackage{amsmath}
\usepackage{amsthm}
\usepackage{amssymb}
\usepackage{stmaryrd}
\usepackage{latexsym}
\usepackage{bm}
\usepackage{xcolor}
\usepackage{graphicx}
\usepackage{caption}
\usepackage{subcaption}
\usepackage{thmtools,thm-restate}
\theoremstyle{plain}

\theoremstyle{definition}

\usepackage{algorithm}
\usepackage{algpseudocode}
\usepackage{cleveref}
\usepackage{enumitem}
\usepackage{multicol}
\usepackage{wrapfig}
\title{Single Episode Policy Transfer in Reinforcement Learning}

\author{Jiachen Yang \\
Georgia Institute of Technology, USA \\
\texttt{jiachen.yang@gatech.edu}
\And
Brenden Petersen \\
Lawrence Livermore National Laboratory, USA \\
\texttt{petersen33@llnl.gov}
\And
Hongyuan Zha \\
Georgia Institute of Technology, USA \\
\texttt{zha@cc.gatech.edu}
\And
Daniel Faissol \\
Lawrence Livermore National Laboratory, USA \\
\texttt{faissol1@llnl.gov} \\
}

\iclrfinalcopy % Uncomment for camera-ready version, but NOT for submission.
\begin{document}

\maketitle

\begin{abstract}
Transfer and adaptation to new unknown environmental dynamics is a key challenge for reinforcement learning (RL).
An even greater challenge is performing near-optimally in a single attempt at test time, possibly without access to dense rewards, which is not addressed by current methods that require multiple experience rollouts for adaptation.
To achieve single episode transfer in a family of environments with related dynamics, we propose a general algorithm that optimizes a probe and an inference model to rapidly estimate underlying latent variables of test dynamics, which are then immediately used as input to a universal control policy.
This modular approach enables integration of state-of-the-art algorithms for variational inference or RL.
Moreover, our approach does not require access to rewards at test time, allowing it to perform in settings where existing adaptive approaches cannot.
In diverse experimental domains with a single episode test constraint, our method significantly outperforms existing adaptive approaches and shows favorable performance against baselines for robust transfer.
\end{abstract}

\section{Introduction}
One salient feature of human intelligence is the ability to perform well in a single attempt at a new task instance, by recognizing critical characteristics of the instance and immediately executing appropriate behavior based on experience in similar instances.
Artificial agents must do likewise in applications where success must be achieved in one attempt and failure is irreversible.
% such as optimally treating a heterogeneous population of patients with adaptive therapies \citep{petersen2019deep} or conducting a robotic rescue mission in new environmental conditions \citep{murphy2008search}.
This problem setting, \textit{single episode transfer}, imposes a challenging constraint in which an agent experiences---and is evaluated on---only \textit{one} episode of a test instance.

As a motivating example, a key challenge in precision medicine is the uniqueness of each patient's response to therapeutics
% due to patient-specific and disease-specific heterogeneities
\citep{hodson2016precision,bordbar2015personalized,whirl2012pharmacogenomics}.
Adaptive therapy is a promising approach that formulates a treatment strategy as a sequential decision-making problem \citep{zhang2017integrating,West476507,petersen2019deep}.
However, heterogeneity among instances may require explicitly accounting for factors that underlie individual patient dynamics. 
For example, in the case of adaptive therapy for sepsis \citep{petersen2019deep}, predicting patient response prior to treatment is not possible. However, differences in patient responses can be observed via blood measurements very early after the onset of treatment \citep{Cockrell2018responders}.
% For example, a key challenge in medicine is that each patient uniquely responds to therapeutics \citep{hodson2016precision}. That is, a therapeutic option that is optimal for one patient, may be sub-optimal, or even detrimental, to another \citep{bordbar2015personalized} due to patient-specific and/or disease-specific heterogeneities \cite{whirl2012pharmacogenomics}. Elucidating these heterogeneities is often not possible \textit{a-priori}, causing treatment plans to incorporate significant trial-and-error aspects \cite{ginsburg2001personalized} that are sometimes too slow to be effective \citep{aspinall2007realizing}. Adaptive therapies are a new and promising direction in biomedical research in which therapeutic strategies are considered as a sequential decision making problem exploiting the patient's response in a feedback loop \citep{zhang2017integrating, West476507, petersen2019deep}. 
%and must often be inferred during the course of treatment 

As a first step to address \textit{single episode transfer} in reinforcement learning (RL), we propose a general algorithm for near-optimal test-time performance in a family of environments where differences in dynamics can be ascertained early during an episode.
Our key idea is to train an inference model and a probe that together achieve rapid inference of latent variables---which account for variation in a family of similar dynamical systems---using a small fraction (e.g., 5\%) of the test episode, then deploy a universal policy conditioned on the estimated parameters for near-optimal control on the new instance.
% Specifically, we approach this problem by estimating latent variables that account for variation in a family of similar dynamical systems and deploying an optimal policy conditioned on the estimate for a new instance.
%We impose the significant challenge, motivated by real-world applications, that the trained model has access to---and is evaluated on---\textit{only one} episode of a new test instance, which mandates acting near-optimally as soon as possible upon deployment.
%Crucially different from traditional transfer learning in RL \citep{taylor2009transfer}, we impose the significantly more difficult constraint that the agent experiences---and is evaluated on---only one episode in the test instance.
% Our key idea is to train an inference model and a probe that together achieve rapid inference of latent variables using a small fraction (e.g., 5\%) of the test episode, along with a universal policy that conditions upon the estimated parameters for near-optimal control.
Our approach combines the advantages of robust transfer and adaptation-based transfer, as we learn a single universal policy that requires no further training during test, but which is adapted to the new environment by conditioning on an unsupervised estimation of new latent dynamics.

% \footnote{Modularity allows casting a family of environments into a one augmented environment that is compatible with any RL workflow.}
% \bp{Around here could be a time to point out that the method can be abstracted/encapsulated into the environment, even at the software level. I'm not sure the best way to make this point. Many might gloss over this, but I think it's a huge advantage the method for those that see its value.}
In contrast to methods that quickly adapt or train policies via gradients during test but assume access to multiple test rollouts and/or dense rewards \citep{finn2017model,killian2017robust,rakelly2019efficient}, we explicitly optimize for performance in one test episode without accessing the reward function at test time.
Hence our method applies to real-world settings in which rewards during test are highly delayed or even completely inaccessible---e.g., a reward that depends on physiological factors that are accessible only in simulation and not from real patients.
We also consider computation time a crucial factor for real-time application, whereas some existing approaches require considerable computation during test \citep{killian2017robust}. Our algorithm builds on variational inference and RL as submodules, which ensures practical compatibility with existing RL workflows.

% paragraph summarizing contributions
Our main contribution is a simple general algorithm for single episode transfer in families of environments with varying dynamics, via rapid inference of latent variables and immediate execution of a universal policy.
Our method attains significantly higher cumulative rewards, with orders of magnitude faster computation time during test, than the state-of-the-art model-based method \citep{killian2017robust}, on benchmark high-dimensional domains whose dynamics are discontinuous and continuous in latent parameters.
We also show superior performance over optimization-based meta-learning and favorable performance versus baselines for robust transfer.
%Extensive ablation experiments on four variants of our method provide empirical support for our algorithm design.
% Code for all experiments is available at \url{https://github.com/011235813/sept}.

\vspace{-3pt}
\section{Single episode transfer in RL: problem setup}

% \vspace{-10pt}
Our goal is to train a model that performs close to optimal within a single episode of a test instance with new unknown dynamics. 
% drawn from the same distribution as training instances.
% \bp{Perhaps introduce a line here up front as to why we have this goal, or more specifically, why we don't allow for multiple episodes to transfer as in few-shot learning.}
% In this section, we formalize this problem in the framework of hidden-parameter Markov decision process (HiP-MDP) and specify the general setup for training and test.
We formalize the problem as a family $( \Scal, \Acal, \Tcal, R, \gamma )$, where $( \Scal, \Acal, R, \gamma)$ are the state space, action space, reward function, and discount of an episodic Markov decision process (MDP).
Each \textit{instance} of the family is a stationary MDP with transition function $\Tcal_z(s'|s,a) \in \Tcal$.
When a set $\Zcal$ of physical parameters determines transition dynamics \citep{konidaris2014hidden}, each $\Tcal_z$ has a hidden parameter $z \in \Zcal$ that is sampled once from a distribution $P_{\Zcal}$ and held constant for that instance.
For more general stochastic systems whose modes of behavior are not easily attributed to physical parameters, $\Zcal$ is induced by a generative latent variable model that indirectly associates each $\Tcal_z$ to a latent variable $z$ learned from observed trajectory data.
We refer to ``latent variable'' for both cases, with the clear ontological difference understood.
Depending on application, $\Tcal_z$ can be continuous or discontinuous in $z$.
We strictly enforce the challenging constraint that latent variables are never observed, in contrast to methods that use known values during training \citep{Yu-RSS-17},
to ensure the framework applies to challenging cases without prior knowledge.

% \textbf{HiP-MDP} \citep{konidaris2014hidden,doshi2016hidden}.
% An HiP-MDP is a family $\lbrace S, A, Z, T, R, \gamma, P_{Z} \rbrace$ where $\lbrace S,A,R,\gamma \rbrace$ are the usual state space, action space, reward function, and discount of an episodic Markov decision process (MDP).
% Each \textit{instance} of an HiP-MDP is a stationary MDP whose transition function $T_z(s'|s,a)$ is determined by an additional hidden parameter $z \in Z$, 
% % \bp{$Z$ itself can be discrete (common example is Boolean hidden parameter). The main point is that $z$ is multi-dimensional.}
% which is sampled once from distribution $P_Z$ and held constant for that instance.
% Depending on the application, the transition function $T_z$ can be continuous in the argument $z$ when the domain $Z \subset R^D$ is continuous, or a finite set of $\lbrace T_z \rbrace$ is induced by a corresponding discrete set $Z$.

This formulation captures a diverse set of important problems.
Latent space $\Zcal$ has physical meaning in systems where $\Tcal_z$ is a continuous function of physical parameters (e.g., friction and stiffness) with unknown values.
In contrast, a discrete set $\Zcal$ can induce qualitatively different dynamics, such as a 2D navigation task where $z \in \lbrace 0,1 \rbrace$ decides if the same action moves in either a cardinal direction or its opposite \citep{killian2017robust}.
Such drastic impact of latent variables may arise when a single drug is effective for some patients but causes serious side effects for others \citep{Cockrell2018responders}.

\textbf{Training phase.} 
Our training approach is fully compatible with
RL for episodic environments.
% standard procedure in episodic RL.
We sample many instances, either via a simulator with controllable change of instances or using off-policy batch data in which demarcation of instances---but not values of latent variables---is known, and train for one or more episodes on each instance.
While we focus on the case with known change of instances, the rare case of unknown demarcation can be approached either by preprocessing steps such as clustering trajectory data or using a dynamic variant of our algorithm (\Cref{app:dynasept}).

\textbf{Single test episode.}
In contrast to prior work that depend on the luxury of multiple experience rollouts for adaptation during test time \citep{doshi2016hidden,killian2017robust,finn2017model,rakelly2019efficient}, we introduce the strict constraint that the trained model has access to---and is evaluated on---\textit{only one} episode of a new test instance.
This reflects the need to perform near-optimally as soon as possible in critical applications such as precision medicine, where an episode for a new patient with new physiological dynamics is the entirety of hospitalization.

\section{Single episode policy transfer}

We present Single Episode Policy Transfer (SEPT), a high-level algorithm for single episode transfer between MDPs with different dynamics.
The following sections discuss specific design choices in SEPT, all of which are combined in synergy for near-optimal performance in a single test episode.

\subsection{Policy transfer through latent space}
\label{subsec:transfer}

% \jc{Decide on whether to call it ``estimation'' or ``learning latent representation''. Estimation suggests that we need to recover the numerical values of the true HP.}
% \hz{Use something similar to estimation will be useful as this will distinguish the work from the crowded trajectory embedding area, but do elaborate on the close connections.}
Our best theories of natural and engineered systems involve physical constants and design parameters that enter into dynamical models.
This physicalist viewpoint motivates a partition for transfer learning in families of MDPs:
1. learn a representation of latent variables with an inference model that rapidly encodes a vector $\hat{z}$ of discriminative features for a new instance; 2. train a universal policy $\pi(a|s,z)$ to perform near-optimally for dynamics corresponding to any latent variable in $\Zcal$; 3. immediately deploy both the inference model and universal policy on a given test episode.
To build on the generality of model-free RL, and for scalability to systems with complex dynamics, we do not expend computational effort to learn a model of $\Tcal_z(s'|s,a)$, in contrast to model-based approaches \citep{killian2017robust,yao2018direct}.
Instead, we leverage expressive variational inference models to represent latent variables and provide uncertainty quantification.

In domains with ground truth hidden parameters, a latent variable encoding is the most succinct representation of differences in dynamics between instances.
As the encoding $\zhat$ is held constant for all episodes of an instance, a universal policy $\pi(a|s,z)$ can either adapt to all instances when $\Zcal$ is finite, or interpolate between instances when $\Tcal_z$ is continuous in $z$ \citep{schaul2015universal}.
Estimating a discriminative encoding for a new instance enables immediate deployment of $\pi(a|s,z)$ on the single test episode, bypassing the need for further fine-tuning.
% For example, in the 2D navigation example above, distinguishing between $z=0$ and $z=1$ allows the agent to apply the exact optimal policy.
This is critical for applications where further training complex models on a test instance is not permitted due to safety concerns.
In contrast, methods that do not explicitly estimate a latent representation of varied dynamics must use precious experiences in the test episode to tune the trained policy \citep{finn2017model}.
% Moreover, latent estimates can provide additional insight on test instances for downstream tasks, such as clustering a patient population.

In the training phase, we generate an optimized\footnote{In the sense of machine teaching, as explained fully in \Cref{subsec:probe}}
dataset $\Dcal := \lbrace \tau^i \rbrace_{i=1}^N$ of short trajectories, where each $\tau^i := (s^i_1,a^i_1,\dotsc,s^i_{T_p},a^i_{T_p})$
is a sequence of early state-action pairs at the start of episodes of instance $\Tcal_i \in \Tcal$ (e.g. $T_p=5$).
% is a sequence of state-action pairs from instance $\Tcal_i \in \Tcal$.
% , using sampled instances from the HiP-MDP with unknown hidden parameters $\lbrace z^i \rbrace_{i=1}^N$.
%\hz{$z_i$ known or unknown?}
%\jc{The true $z$ is always unknown.}
% \bp{Here (and a few times elsewhere) you use subscript $i$ to denote the $i$th hidden parameter for both $z$ and $\tau$, but then switch to superscript. Just needs to be consistent. Since subscripts sometimes denote time step, perhaps superscript?}
% \jc{I fixed it. Just to be clear, $z^i$ is the entire hidden parameter vector.}
We train a variational auto-encoder, comprising an approximate posterior inference model $q_{\phi}(z|\tau)$ that produces a latent encoding $\hat{z}$ from $\tau$ and a parameterized generative model $p_{\psi}(\tau|z)$.
The dimension chosen for $\hat{z}$ may differ from the exact true dimension when it exists but is unknown; domain knowledge can aid the choice of dimensionality reduction.
Because dynamics of a large variety of natural systems are determined by independent parameters (e.g., coefficient of contact friction and Reynolds number can vary independently),
we consider a disentangled latent representation where latent units capture the effects of independent generative parameters.
To this end, we bring $\beta$-VAE \citep{higgins2017beta} into the context of families of dynamical systems, choosing an isotropic unit Gaussian as the prior and imposing the constraint $D_{KL}(q_{\phi}(z|\tau^i)\Vert p(z)) < \epsilon$.
The $\beta$-VAE is trained by maximizing the variational lower bound $\Lcal(\psi,\phi;\tau^i)$ for each $\tau^i$ across $\Dcal$:
\begin{align}\label{eq:variational-lowerbound}
    \max_{\psi,\phi} \log p_{\psi}(\tau^i) \geq \Lcal(\psi,\phi;\tau^i) := -\beta D_{KL}(q_{\phi}(z|\tau^i) \Vert p(z)) + \Ebb_{q_{\phi}(z|\tau^i)} \bigl[ \log p_{\psi}(\tau^i|z) \bigr]
\end{align}
This subsumes the VAE \citep{kingma2013auto} as a special case ($\beta=1$), and we refer to both as VAE in the following.
% We also note that more flexible black-box variational inference methods such as AVB \citep{mescheder2017adversarial} can be readily substituted into our general algorithm.
Since latent variables only serve to differentiate among trajectories that arise from different transition functions, the meaning of latent variables is not affected by isometries and hence the value of $\zhat$ by itself need not have any simple relation to a physically meaningful $z$ even when one exists.
Only the partition of latent space is important for training a universal policy.
% \hz{They may ask you what $\beta$ value you used and how that affected performance. Can you assign physical meaning to the learned $z$? Can you recover the physically meaningful $z$ if there's one by tuning $\beta$?}

% Earlier methods for HiP-MDP 
Earlier methods for a family of similar dynamics
relied on Bayesian neural network (BNN) approximations of the entire transition function $s_{t+1} \sim \hat{\Tcal}^{(\textrm{BNN})}_z(s_t,a_t)$, which was either used to perform computationally expensive fictional rollouts during test time \citep{killian2017robust} or used indirectly to further optimize a posterior over $z$ \citep{yao2018direct}.
Our use of variational inference is more economical: the encoder $q_{\phi}(z|\tau)$ can be used immediately to infer latent variables during test, while the decoder $p_{\psi}(\tau|z)$ plays a crucial role for optimized probing in our algorithm (see \Cref{subsec:probe}).

In systems with ground truth hidden parameters, we desire two additional properties.
The encoder should produce low-variance encodings, which we implement by minimizing the entropy of $q_{\phi}(z|\tau)$:
\begin{align}\label{eq:entropy-encoder}
    \min_{\phi} H(q_{\phi}(z|\tau)) &:= - \int_z q_{\phi}(z|\tau) \log q_{\phi}(z|\tau) dz = \frac{D}{2} \log(2\pi) + \frac{1}{2}\sum_{d=1}^D \bigl( 1 + \log \sigma_d^2 \bigr)
\end{align}
under a diagonal Gaussian parameterization, where $\sigma_d^2 = \Var(q_{\phi}(z|\tau))$ and $\dim(z) = D$.
We add $-H(q_{\phi}(z|\tau))$ as a regularizer to \eqref{eq:variational-lowerbound}.
Second, we must capture the impact of $z$ on higher-order dynamics.
While previous work neglects the order of transitions $(s_t,a_t,s_{t+1})$ in a trajectory \citep{rakelly2019efficient}, we note that a single transition may be compatible with multiple instances whose differences manifest only at higher orders.
In general, partitioning the latent space requires taking the ordering of a temporally-extended trajectory into account.
Therefore, we parameterize our encoder $q_{\phi}(z|\tau)$ using a bidirectional LSTM---as both temporal directions of $(s_t,a_t)$ pairs are informative---and we use an LSTM decoder $p_{\psi}(\tau|z)$ (architecture in \Cref{app:architecture}).
In contrast to embedding trajectories from a \textit{single} MDP for hierarchical learning \citep{co2018self}, our purpose is to encode trajectories from \textit{different instances} of transition dynamics for optimal control.

\subsection{Transfer of a universal policy}

We train a single universal policy $\pi(a|s,z)$ and deploy the same policy during test (without further optimization), for two reasons: robustness against imperfection in latent variable representation and significant improvement in scalability.
Earlier methods trained multiple optimal policies $\lbrace \pi^*_i(a|s) \rbrace_{i=1}^N$ on training instances with a set $\lbrace z^i \rbrace_{i=1}^N$ of hidden parameters, then employed either behavioral cloning \citep{yao2018direct} or off-policy Q-learning \citep{arnekvist2019vpe} to train a final policy $\pi(a|s,z)$ using a dataset $\lbrace (s_t, \hat{z}^i; a_t\sim \pi^*_i(a|s_t)) \rbrace$.
% , where encoding $\hat{z}^i$ and state $s_t$ are inputs and target actions are sampled from the corresponding optimal policy $\pi^*_i(a|s_t)$.
However, this supervised training scheme may not be robust \citep{Yu-RSS-17}: if $\pi(a|s,z)$ were trained only using instance-specific optimal state-action pairs generated by $\pi^*_i(a|s)$ and posterior samples of $\hat{z}$ from an optimal inference model,
it may not generalize well when faced with states and encodings that were not present during training.
Moreover, it is computationally infeasible to train a collection $\lbrace \pi^*_i \rbrace_{i=1}^N$---which is thrown away during test---when faced with a large set of training instances from a continuous set $\Zcal$.
Instead, we interleave training of the VAE and a single policy $\pi(a|s,z)$, benefiting from considerable computation savings at training time, and higher robustness due to larger effective sample count.
% In our experiments, we use DDQN with prioritized replay \citep{van2016deep,schaul2015prioritized}, while other RL algorithm can be readily substituted (e.g., PPO for continuous action spaces \citep{schulman2017proximal}).
% \footnote{For DDQN, the induced policy is $\pi(a|s,z) = 1$ for $a = \argmax_a Q(s,a,z)$, else $\pi(a|s,z) = 0$.},
% and PPO \citep{schulman2017proximal} for discrete and continuous action spaces, respectively, 

\subsection{Optimized probing for accelerated latent variable inference}
\label{subsec:probe}

%\hz{There should be several desirable properties for the $z$: 1. disentangled; 2. easy to learn from short trajectories; 3. easy/feasible to generate short trajectories to learn it. The above could be in conflict, and we have to compromise.}
%\jc{Added mention of disentangled representation above in Section 3.1 third paragraph.}

To execute near-optimal control within a single test episode, we first rapidly compute $\hat{z}$ using a short trajectory of initial experience.
This is loosely analogous to the use of preliminary medical treatment to define subsequent prescriptions that better match a patient's unique physiological response.
% While human experts may depend on domain knowledge to decide initial treatment, 
Our goal of rapid inference
motivates two algorithmic design choices to optimize this initial phase.
First, the trajectory $\tau$ used for inference by $q_{\phi}(z|\tau)$ must be optimized, in the sense of machine teaching \citep{zhu2018overview}, as certain trajectories are more suitable than others for inferring latent variables that underlie system dynamics.
If specific degrees of freedom are impacted the most by latent variables, an agent should probe exactly those dimensions to produce an informative trajectory for inference.
Conversely, methods that deploy a single universal policy without an initial probing phase \citep{yao2018direct} can fail in adversarial cases, such as when the initial placeholder $\zhat$ used in $\pi_{\theta}(a|s,\cdot)$ at the start of an instance causes failure to exercise dimensions of dynamics that are necessary for inference.
Second, the VAE must be specifically trained on a dataset $\Dcal$ of short trajectories consisting of initial steps of each training episode.
We cannot expend a long trajectory for input to the encoder during test, to ensure enough remaining steps for control.
Hence, single episode transfer motivates the machine teaching problem of learning to distinguish among dynamics: our algorithm must have learned both to generate and to use a short initial trajectory to estimate a representation of dynamics for control.

Our key idea of optimized probing for accelerated latent variable inference is to train a dedicated probe policy $\pi_{\varphi}(a|s)$ to generate a dataset $\Dcal$ of short trajectories at the beginning of all training episodes, such that the VAE's performance on $\Dcal$ is optimized\footnote{In general, $\Dcal$ is not related to the replay buffer commonly used in off-policy RL algorithms.}.
Orthogonal to training a meta-policy for faster exploration \textit{during} standard RL training \citep{xu2018learning}, 
our probe and VAE are trained for the purpose of performing well on a \textit{new} test MDP.
For ease of exposition, we discuss the case with access to a simulator, but our method easily allows use of off-policy batch data.
We start each training episode using $\pi_{\varphi}$ for a \textit{probe phase} lasting $T_p$ steps, record the probe trajectory $\tau_p$ into $\Dcal$, train the VAE using minibatches from $\Dcal$, then use $\tau_p$ with the encoder to generate $\hat{z}$ for use by $\pi_{\theta}(a|s,\hat{z})$ to complete the remainder of the episode (\Cref{alg:sept}).
At test time, SEPT only requires lines 5, 8, and 9 in \Cref{alg:sept} (training step in 9 removed; see \Cref{alg:sept-test}).
The reward function for $\pi_{\varphi}$ is defined as the VAE objective, approximated by the variational lower bound (\ref{eq:variational-lowerbound}): $R_p(\tau) := \Lcal(\psi,\phi;\tau) \leq \log p_{\psi}(\tau)$.
This feedback loop between the probe and VAE directly trains the probe to help the VAE's inference of latent variables that distinguish different dynamics (\Cref{fig:architecture}).
We provide detailed justification as follows.
First we state a result derived in \Cref{app:proof-exploration}:
\begin{restatable}{proposition}{propexploration}
\label{prop:exploration-gradient}
Let $p_{\varphi}(\tau)$ denote the distribution of trajectories induced by $\pi_{\varphi}$.
Then the gradient of the entropy $H(p_{\varphi}(\tau))$ is given by
\begin{align}\label{eq:exploration-gradient}
    \nabla_{\varphi} H(p_{\varphi}(\tau)) &= \Ebb_{p_{\varphi}(\tau)} \bigl[ \nabla_{\varphi} \sum_{i=0}^{T_p-1} \log(\pi_{\varphi}(a_i|s_i)) (-\log p_{\varphi}(\tau)) \bigr]
\end{align}
\end{restatable}
Noting that dataset $\Dcal$ follows distribution $p_{\varphi}$
and that the VAE is exactly trained to maximize the log probability of $\Dcal$,
% and that $p_{\psi}(\tau|z)$ is trained using $\Dcal$ to minimize reconstruction loss of trajectories, 
we use $\Lcal(\psi,\phi;\tau)$ as a tractable lowerbound on $\log p_{\varphi}(\tau)$.
% we approximate $\log p_{\varphi}(\tau)$ using the variational lowerbound $\Lcal(\psi,\phi;\tau)$.
Crucially, to generate optimal probe trajectories for the VAE, we take a minimum-entropy viewpoint and \textit{descend} the gradient (\ref{eq:exploration-gradient}).
This is opposite of a maximum entropy viewpoint that encourages the policy to generate diverse trajectories \citep{co2018self}, which would minimize $\log p_{\varphi}(\tau)$ and produce an adversarial dataset for the VAE---hence, optimal probing is not equivalent to diverse exploration.
The degenerate case of $\pi_{\varphi}$ learning to ``stay still'' for minimum entropy is precluded by any source of environmental stochasticity: trajectories from different instances will still differ, so degenerate trajectories result in low VAE performance.
Finally we observe that \eqref{eq:exploration-gradient} is the defining equation of a simple policy gradient algorithm \citep{williams1992simple} for training $\pi_{\varphi}$, with $\log p_{\varphi}(\tau)$ interpreted as the cumulative reward of a trajectory generated by $\pi_{\varphi}$.
This completes our justification for defining reward $R_p(\tau) := \Lcal(\psi,\phi;\tau)$.
We also show empirically in ablation experiments that this reward is more effective than choices that encourage high perturbation of state dimensions or high entropy (\Cref{sec:results}).
\begin{wrapfigure}{r}{0.55\textwidth}
  \centering
    \includegraphics[width=0.55\textwidth]{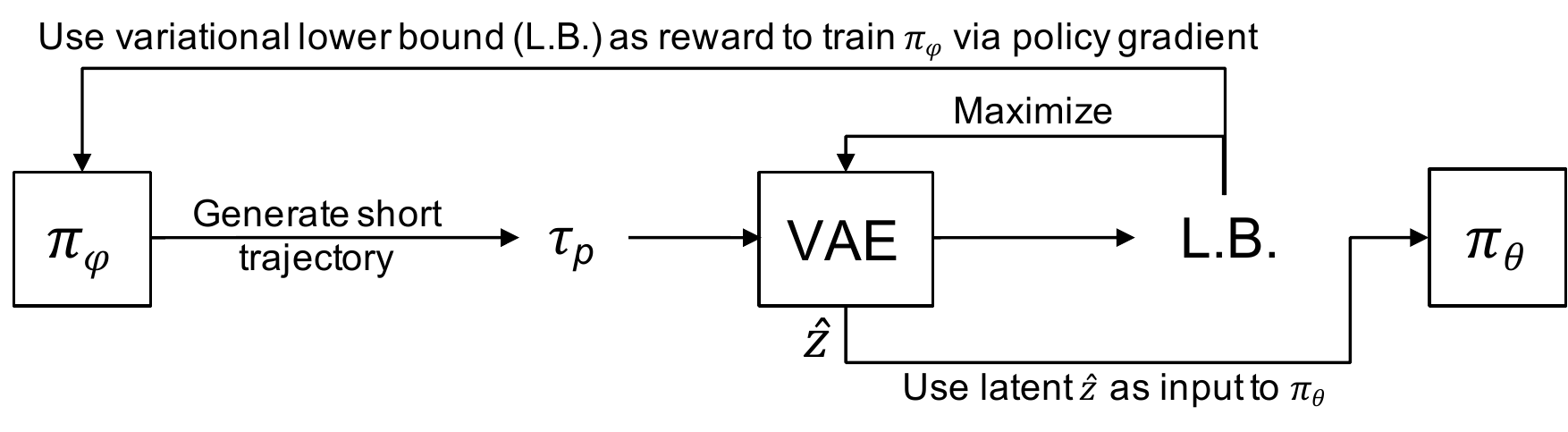}
  \caption{$\pi_{\varphi}$ learns to generate an optimal dataset for the VAE, whose performance is the reward for $\pi_{\varphi}$. Encoding $\zhat$ by the VAE is given to control policy $\pi_{\theta}$.}
  \label{fig:architecture}
\end{wrapfigure}

The VAE objective function may not perfectly evaluate a probe trajectory generated by $\pi_{\varphi}$ because the objective value increases due to VAE training regardless of $\pi_{\varphi}$.
To give a more stable reward signal to $\pi_{\varphi}$, we can use a second VAE whose parameters slowly track the main VAE according to $\psi' \leftarrow \alpha \psi + (1-\alpha) \psi'$ for $\alpha \in [0,1]$, and similarly for $\phi'$.
While analogous to target networks in DQN \citep{mnih2015human}, the difference is that our second VAE is used to compute the \textit{reward} for $\pi_{\varphi}$.

\begin{algorithm}[t]
\caption{Single Episode Policy Transfer: training phase}
\label{alg:sept}
\begin{algorithmic}[1]
\footnotesize
\Procedure{SEPT-train}{}
\State Initialize encoder $\phi$, decoder $\psi$, probe policy $\varphi$, control policy $\theta$, and trajectory buffer $\Dcal$
\For{each instance $\Tcal_z$ with transition function sampled from $\Tcal$}
    \For{each episode on instance $\Tcal_z$}
        \State Execute $\pi_{\varphi}$ for $T_p$ steps and store trajectory $\tau_p$ into $\Dcal$
        \State Use variational lower bound (\ref{eq:variational-lowerbound}) as the reward to train $\pi_{\varphi}$ by descending gradient (\ref{eq:exploration-gradient})
        \State Train VAE using minibatches from $\Dcal$  for gradient ascent on (\ref{eq:variational-lowerbound}) and descent on (\ref{eq:entropy-encoder})
        \State Estimate $\hat{z}$ from $\tau_p$ using encoder $q_{\phi}(z|\tau)$
        \State Execute $\pi_{\theta}(a|s,z)$ with $\hat{z}$ for remaining time steps and train it with suitable RL algorithm
\EndFor
\EndFor
\EndProcedure
\end{algorithmic}
\end{algorithm}

\section{Related work}

Transfer learning in a family of MDPs with different dynamics manifests in various formulations \citep{taylor2009transfer}.
Analysis of $\epsilon$-stationary MDPs and $\epsilon$-MDPs
% , involving a collection of transition functions with bounded variation, 
provide theoretical grounding by showing that an RL algorithm that learns an optimal policy in an MDP can also learn a near-optimal policy for multiple transition functions \citep{kalmar1998module,szita2002varepsilon}.

%When physically meaningful ground truth parameters determine each instance's dynamics, the problem can be formalized as a hidden-parameter Markov decision process (HiP-MDP) \citep{konidaris2014hidden}.
Imposing more structure, the hidden-parameter Markov decision process (HiP-MDP) formalism posits a space of hidden parameters that determine transition dynamics, and implements transfer by model-based policy training after inference of latent parameters \citep{doshi2016hidden,konidaris2014hidden}. 
Our work considers HiP-MDP as a widely applicable yet special case of a general viewpoint, in which the existence of hidden parameters is not assumed but rather is induced by a latent variable inference model.
% Differences from general transfer learning, multi-task learning and POMDP
The key structural difference from POMDPs \citep{kaelbling1998planning} is that given fixed latent values, each instance from the family is an MDP with no hidden states; hence, unlike in POMDPs, tracking a history of observations provides no benefit.
In contrast to multi-task learning \citep{caruana1997multitask}, which uses the same tasks for training and test, and in contrast to parameterized-skill learning \citep{da2012learning}, where an agent learns from a collection of rewards with given task identities in one environment with fixed dynamics,
our training and test MDPs have different dynamics and identities of instances are not given.
% our family of MDPs has different dynamics, identities of instances are not given, and the set of test instances is potentially infinite.

Prior latent variable based methods for transfer in RL depend on a multitude of optimal policies during training \citep{arnekvist2019vpe},
or learn a surrogate transition model for model predictive control with real-time posterior updates during test \citep{perez2018efficient}.
Our variational model-free approach does not incur either of these high computational costs.
We encode trajectories to infer latent representation of differing dynamics, in contrast to state encodings in \citep{zhang2018decoupling}.
Rather than formulating variational inference in the space of optimal value functions \citep{tirinzoni2018transfer}, 
% which may limit application to discrete action spaces, 
we implement transfer through variational inference in a latent space that underlies dynamics.
Previous work for transfer across dynamics with hidden parameters employ model-based RL with Gaussian process and Bayesian neural network (BNN) models of the transition function \citep{doshi2016hidden,killian2017robust}, which require computationally expensive fictional rollouts to train a policy from scratch during test time and poses difficulties for real-time test deployment.
DPT uses a fully-trained BNN to further optimize latent variable during a single test episode, but faces scalability issues as it needs one optimal policy per training instance \citep{yao2018direct}.
In contrast, our method does not need a transition function and can be deployed without optimization during test.
Methods for robust transfer either require access to multiple rounds from the test MDP during training \citep{rajeswaran2016epopt}, or require the distribution over hidden variables to be known or controllable \citep{paul2019fingerprint}.
% EPOpt learns a robust policy from a distribution of source MDPs but assumes access to multiple rounds from the test MDP during training \citep{rajeswaran2016epopt}.
While meta-learning \citep{finn2017model,rusu2018meta,zintgraf2018fast,rakelly2019efficient} in principle can take one gradient step during a single test episode, prior empirical evaluation were not made with this constraint enforced, and adaptation during test is impossible in settings without dense rewards.

\section{Experimental setup}
\label{sec:experimental-setup}

% \hz{Should we give the HIV example more prominence in the presentation? Adding more details and motivations, even just in the appendix.}
We conducted experiments on three benchmark domains with diverse challenges to evaluate the performance, speed of reward attainment, and computational time of SEPT versus five baselines in the single test episode\footnote{Code for all experiments is available at \url{https://github.com/011235813/SEPT}.}.
We evaluated four ablation and variants of SEPT to investigate the necessity of all algorithmic design choices.
For each method on each domain, we conducted 20 independent training runs.
For each trained model, we evaluate on $M$ independent test instances, all starting with the same model; adaptations during the single test episode, if done by any method, are not preserved across the independent test instances.
This means we evaluate on a total of $20M$ independent test instances per method per domain.
Hyperparameters were adjusted using a coarse coordinate search on validation performance.
% Any adaptation on one test instance is \textit{not} preserved for subsequent independent test instances;
We used DDQN with prioritized replay \citep{van2016deep,schaul2015prioritized} as the base RL component of all methods for a fair evaluation of transfer performance; other RL algorithms can be readily substituted.

% For each trained model by each method, each test instance starts with the same model; any adaptation is not preserved across test episodes.
% Each instance has stationary dynamics.

\textbf{Domains. }
We use the same continuous state discrete action HiP-MDPs proposed by \citet{killian2017robust} for benchmarking.
Each isolated instance from each domain is solvable by RL, but it is highly challenging, if not impossible, for na\"ive RL to perform optimally for all instances because significantly different dynamics require different optimal policies.
In \textbf{2D navigation}, dynamics are discontinuous in $z \in \lbrace 0, 1\rbrace$ as follows: location of barrier to goal region, flipped effect of actions (i.e., depending on $z$, the \textit{same} action moves in either a cardinal direction or its opposite), and direction of a nonlinear wind.
In \textbf{Acrobot} \citep{sutton1998reinforcement}, the agent applies $\lbrace +1,0,-1\rbrace$ torques to swing a two-link pendulum above a certain height.
Dynamics are determined by a vector $z = (m_1,m_2,l_1,l_2)$ of masses and lengths, centered at 1.0.
We use four unique instances in training and validation, constructed by sampling $\Delta z$ uniformly from $\lbrace -0.3,-0.1,0.1,0.3 \rbrace$ and adding it to all components of $z$.
During test, we sample $\Delta z$ from $\lbrace -0.35,-0.2,0.2,0.35 \rbrace$ to evaluate both interpolation and extrapolation.
% In \textbf{HIV}, 12 hidden variables determine a patient's state dynamics as modeled by differential equations, 
In \textbf{HIV}, a patient's state dynamics is modeled by differential equations with high sensitivity to 12 hidden variables and separate steady-state regions of health, such that different patients require unique treatment policies \citep{adams2004dynamic}.
Four actions determine binary activation of two drugs.
We have $M = 10,5,5$ for 2D navigation, Acrobot, and HIV, respectively.
% We used the same training and test instances as \citet{killian2017robust}.
% Our fourth domain is a new challenging continuous action simulation of \textbf{Cancer}, involving a specific test for extrapolation: with a scaling factor $\alpha \in [0,1]$, we sample each latent dimension $z_i$ during training from $z_i \sim \text{Unif}[ z_{i,min} + \alpha\delta_i, z_{i,max} - \alpha\delta_i ]$
% where $\delta_i := \frac{1}{2}(z_{i,max} - z_{i,min})$;
% during test, we sample from outside the training support $z_i \sim 0.5\text{ Unif}[ z_{i,min}, z_{i,min} + \alpha\delta_i ] + 0.5\text{ Unif}[ z_{i,max} - \alpha\delta_i, z_{i,max} ]$.
% The length of probe trajectories used in SEPT is $T_p=2$ for 2D navigation and $T_p=5$ for all others.

\textbf{Baselines. }
First, we evaluated two simple baselines that establish approximate bounds on test performance of methods that train a single policy:
as a lower bound, \textbf{Avg} trains a single policy $\pi(a|s)$ on all instances sampled during training and runs directly on test instances; 
as an upper bound in the limit of perfect function approximation for methods that use latent variables as input, \textbf{Oracle} $\pi(a|s,z)$ receives the true hidden parameter $z$ during both training and test.
Next we adapted existing methods, detailed in \Cref{app:algorithms}, to single episode test evaluation:
1. we allow the model-based method \textbf{BNN} \citep{killian2017robust} to fine-tune a pre-trained BNN and train a policy using BNN-generated fictional episodes every 10 steps during the test episode; 2. we adapted the adversarial part of EPOpt \citep{rajeswaran2016epopt}, which we term \textbf{EPOpt-adv}, by training a policy $\pi(a|s)$ on instances with the lowest 10-percentile performance; 3.
we evaluate \textbf{MAML} as an archetype of meta-learning methods that require dense rewards or multiple rollouts \citep{finn2017model}.
We allow MAML to use a trajectory of the same length as SEPT's probe trajectory for one gradient step during test.
We used the same architecture for the RL module of all methods (\Cref{app:architecture}).
To our knowledge, these model-free baselines are evaluated on single-episode transfer for the first time in this work.

\textbf{Ablations. }
To investigate the benefit of our optimized probing method for accelerated inference, we designed an ablation called \textbf{SEPT-NP},
in which trajectories generated by the control policy are used by the encoder for inference and stored into $\Dcal$ to train the VAE.
% in which the probe is removed.
% Instead, trajectories generated by the control policy are used by the encoder for inference and stored into $\Dcal$ to train the VAE.
Second, we investigated an alternative reward function for the probe, labeled \textbf{TotalVar} and defined as $R(\tau) := 1/T_p \sum_{t=1}^{T_p-1} \sum_{i=1}^{\text{dim}(\Scal)} |s_{t+1,i} - s_{t,i}|$ for probe trajectory $\tau$.
In contrast to the minimum entropy viewpoint in \Cref{subsec:probe}, this reward encourages generation of trajectories that maximize total variation across all state space dimensions.
Third, we tested the maximum entropy viewpoint on probe trajectory generation, labeled \textbf{MaxEnt}, by giving \textit{negative} lowerbound as the probe reward: $R_p(\tau) := -\Lcal(\psi,\phi;\tau)$.
Last, we tested whether \textbf{DynaSEPT}, an extension that dynamically decides to probe or execute control (\Cref{app:dynasept}), has any benefit for stationary dynamics.

\begin{figure}[t]
\centering
\begin{subfigure}[t]{0.20\linewidth}
    \centering
    \includegraphics[width=1.0\linewidth]{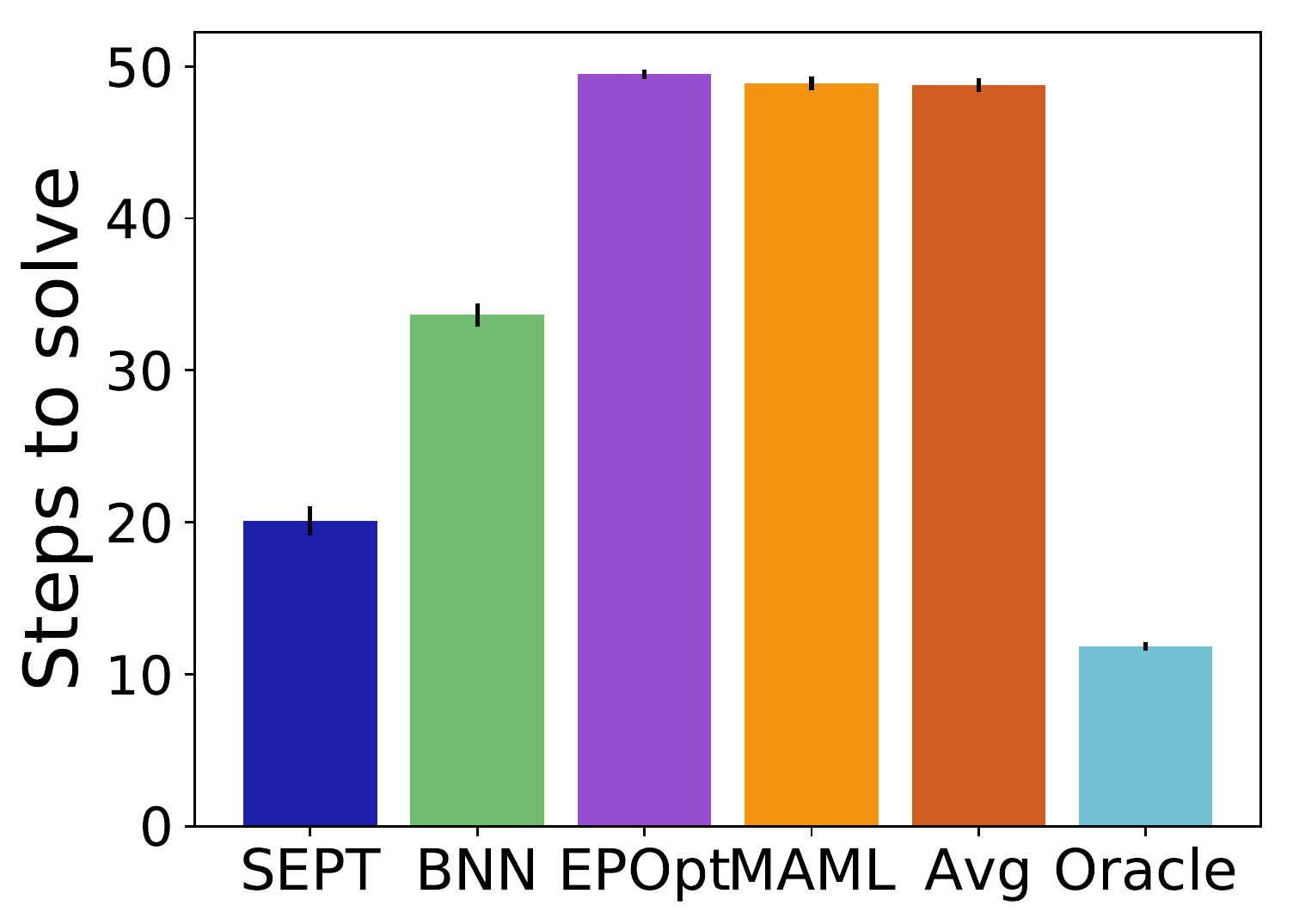}
    \caption{2D navigation}
    \label{fig:2D-steps}
\end{subfigure}
\hfill
\begin{subfigure}[t]{0.20\linewidth}
    \centering
    \includegraphics[width=1.0\linewidth]{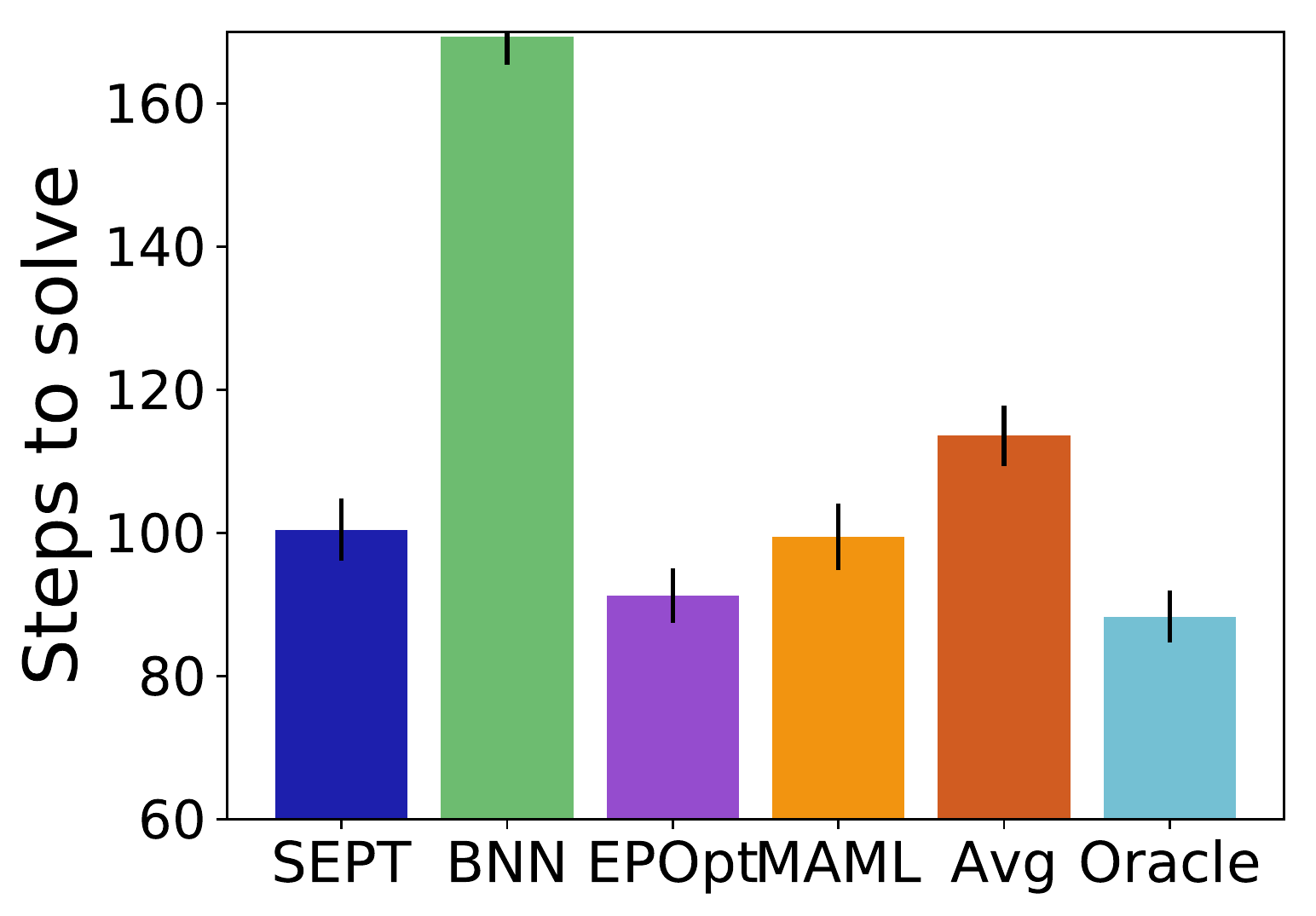}
    \caption{Acrobot}
    \label{fig:acrobot-steps}
\end{subfigure}
\hfill
\begin{subfigure}[t]{0.19\linewidth}
    \centering
    \includegraphics[width=0.95\linewidth]{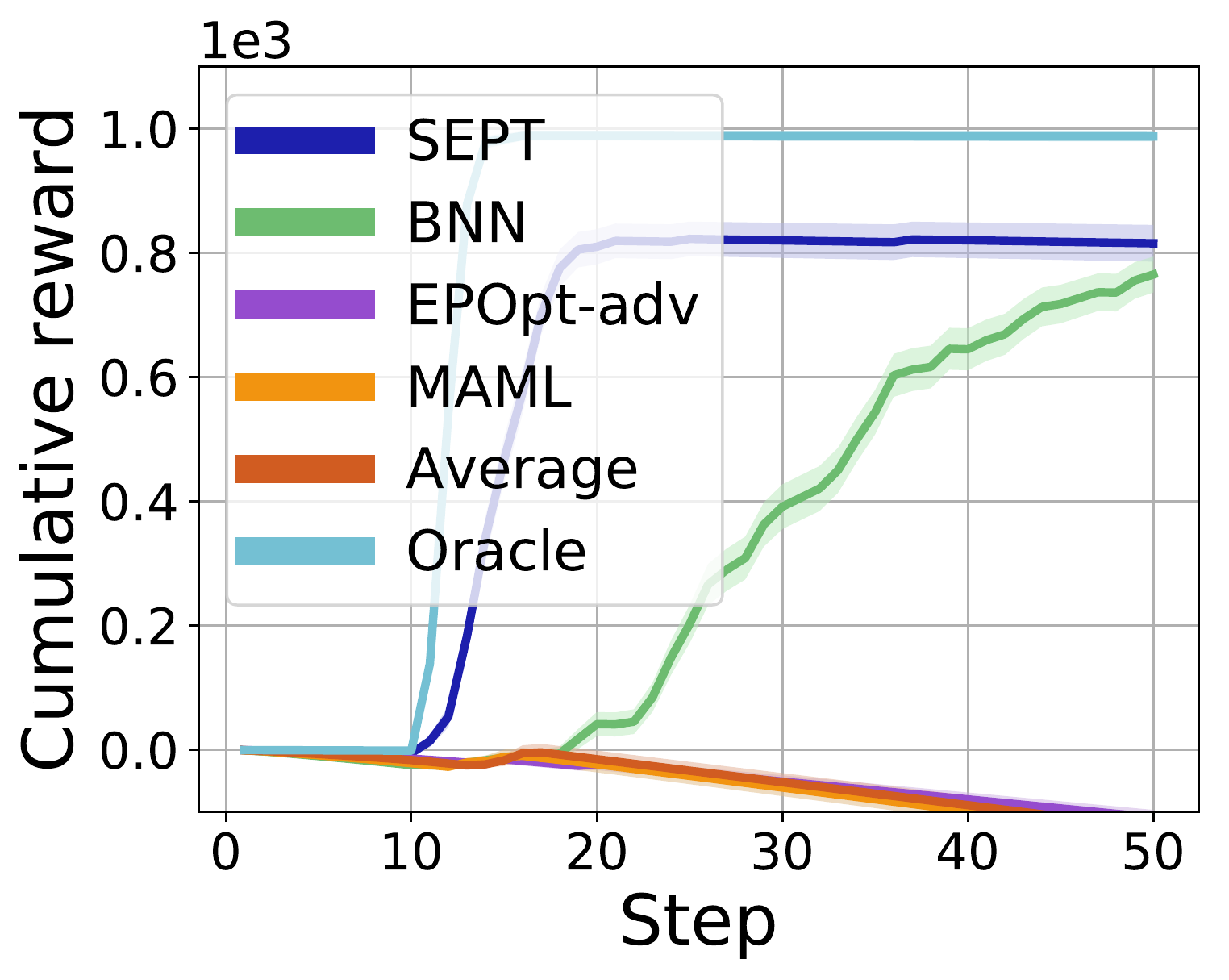}
    \caption{2D navigation}
    \label{fig:2D}
\end{subfigure}
\hfill
\begin{subfigure}[t]{0.19\linewidth}
    \centering
    \includegraphics[width=0.95\linewidth]{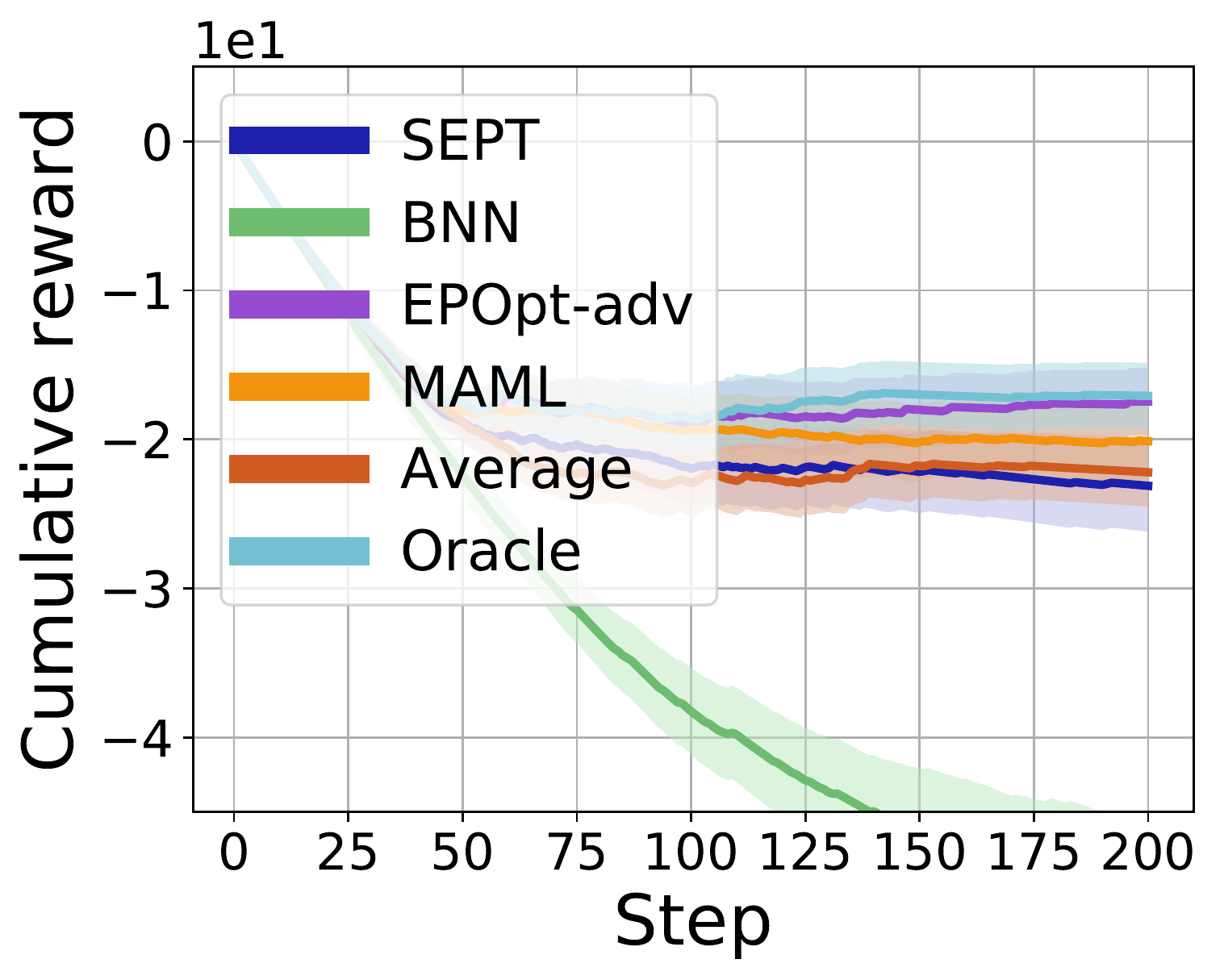}
    \caption{Acrobot}
    \label{fig:acrobot}
\end{subfigure}
\hfill
\begin{subfigure}[t]{0.19\linewidth}
    \centering
    \includegraphics[width=0.95\linewidth]{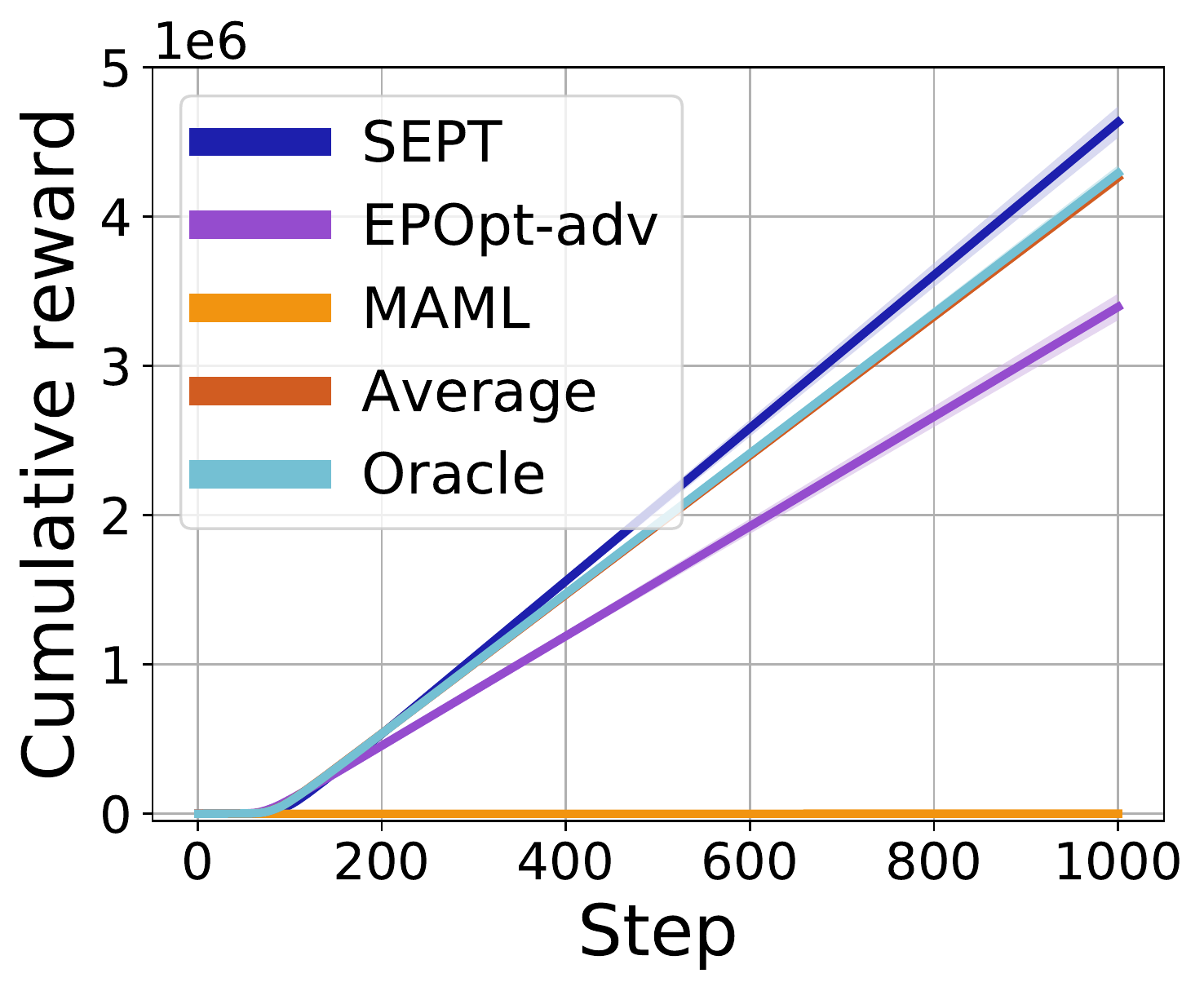}
    \caption{HIV}
    \label{fig:hiv}
\end{subfigure}

\begin{subfigure}[t]{0.20\linewidth}
    \centering
    \includegraphics[width=1.0\linewidth]{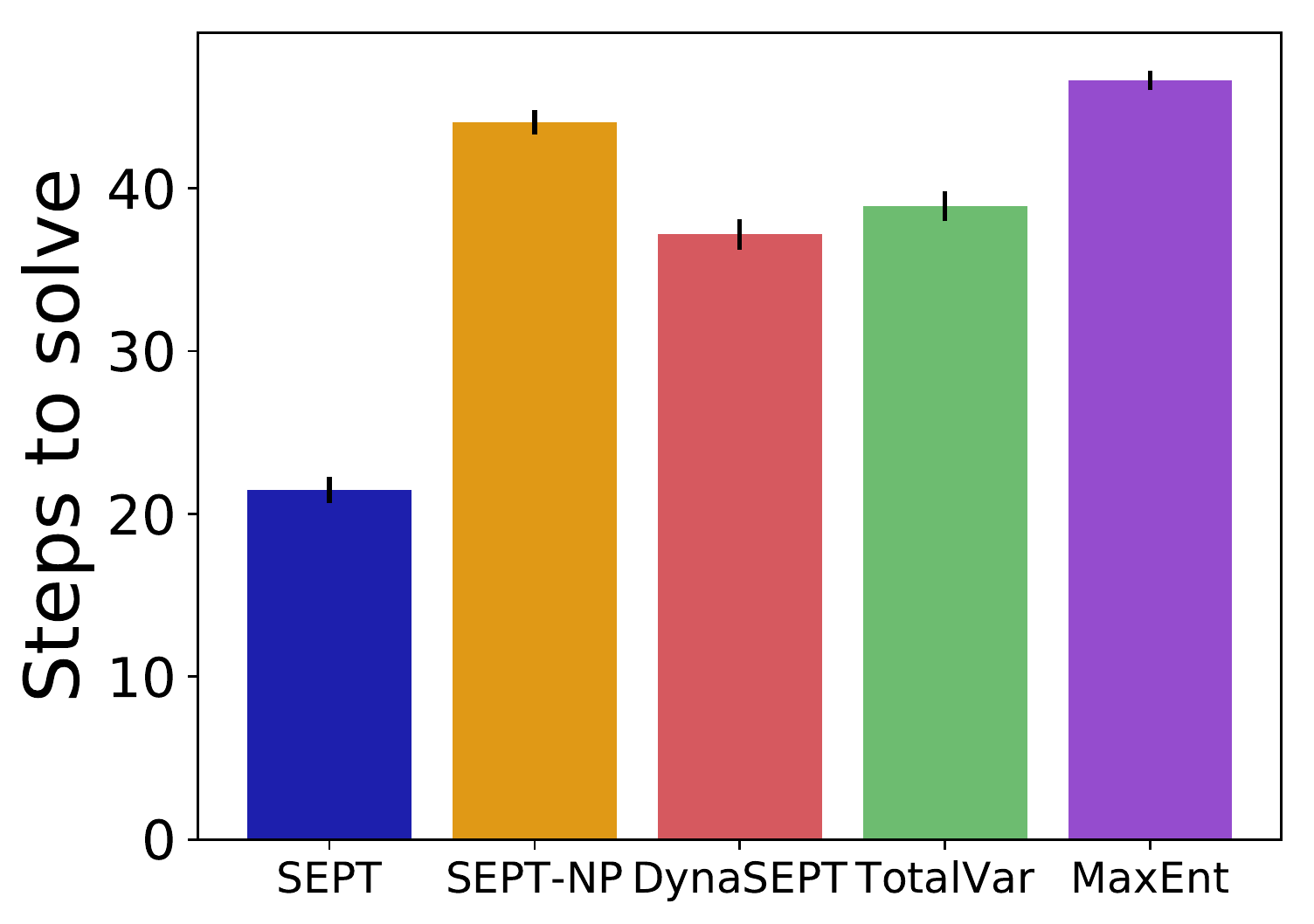}
    \caption{2D navigation}
    \label{fig:2D-ablation-steps}
\end{subfigure}
\hfill
\begin{subfigure}[t]{0.20\linewidth}
    \centering
    \includegraphics[width=1.0\linewidth]{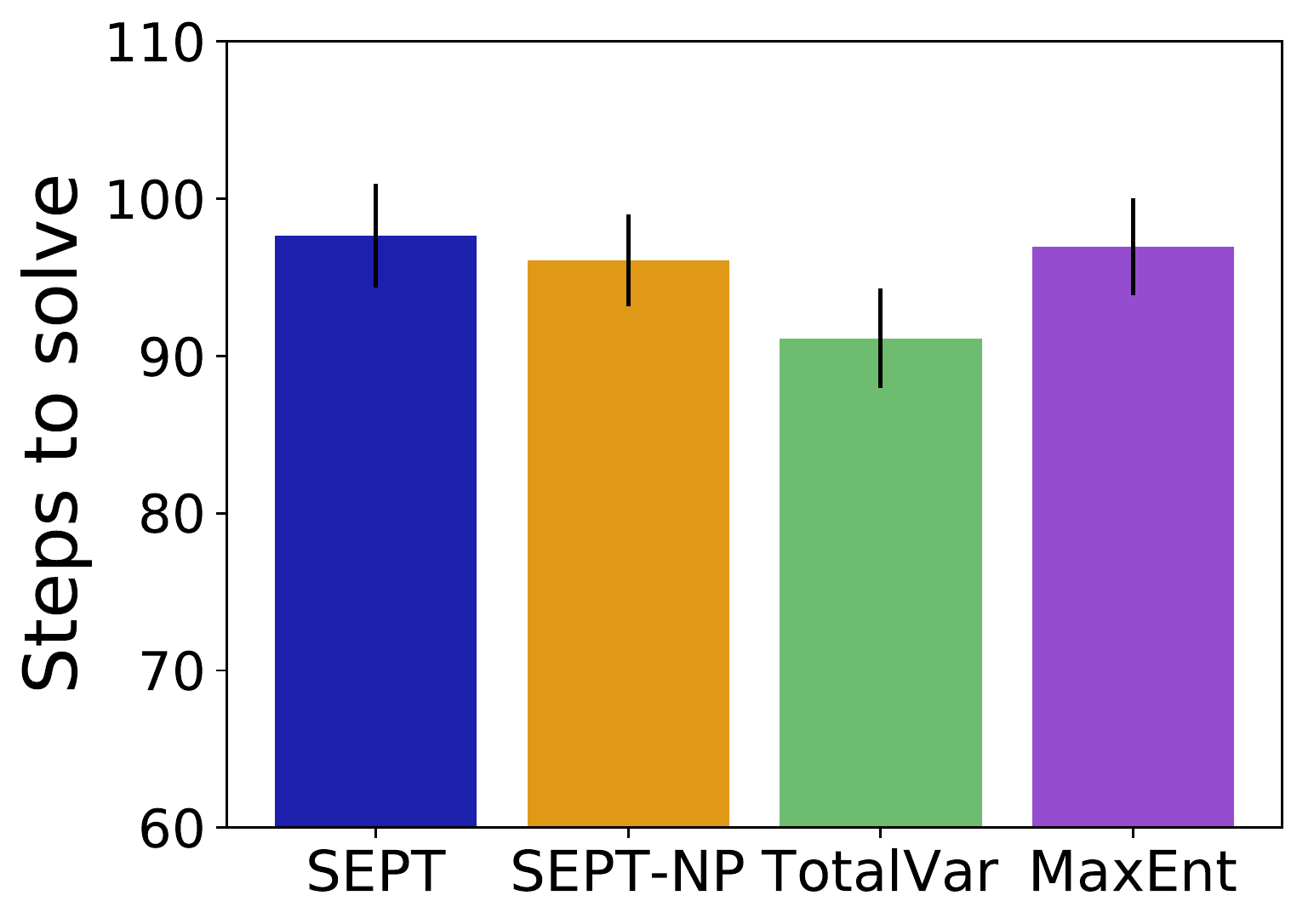}
    \caption{Acrobot}
    \label{fig:acrobot-ablation-steps}
\end{subfigure}
\hfill
\begin{subfigure}[t]{0.19\linewidth}
    \centering
    \includegraphics[width=0.95\linewidth]{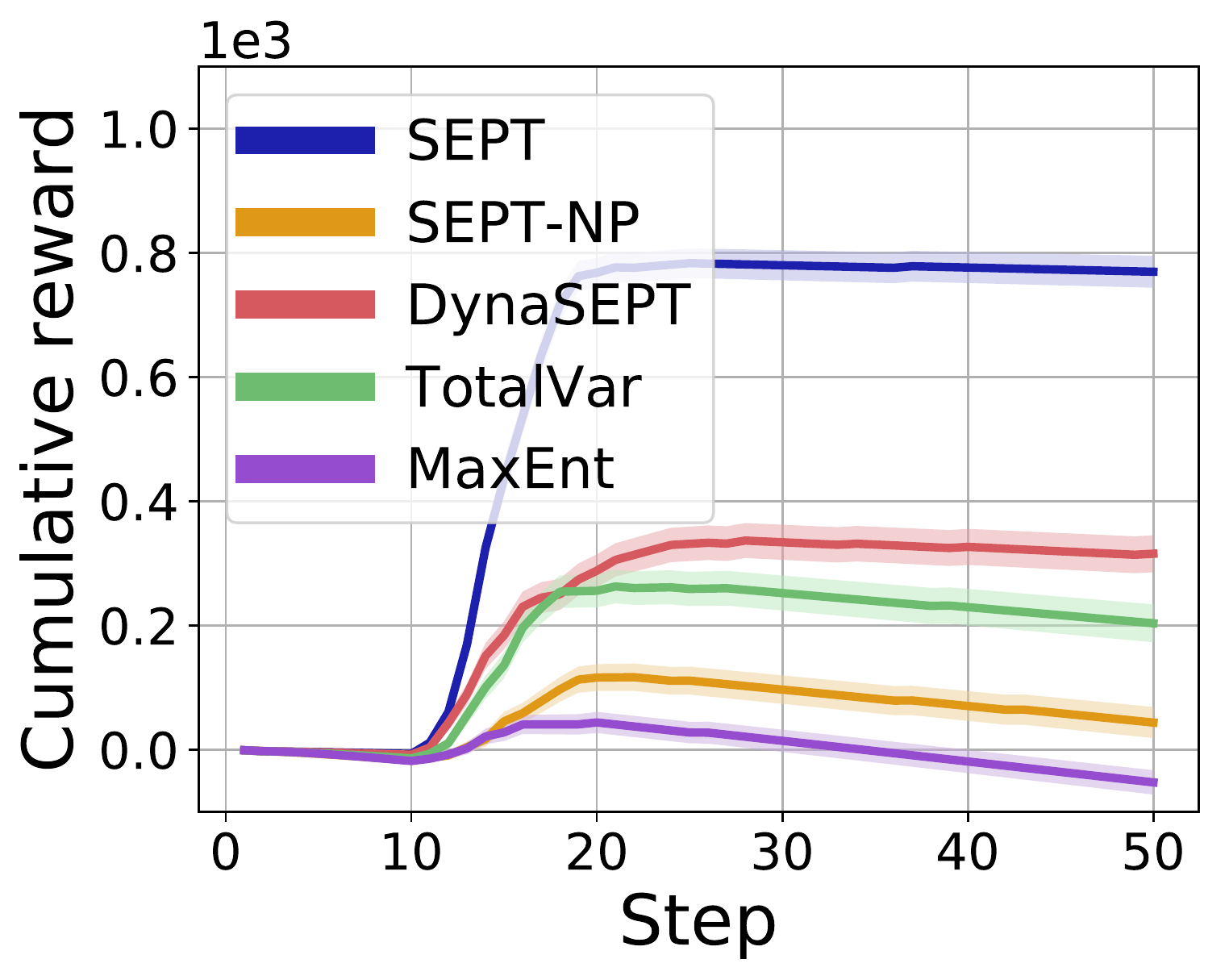}
    \caption{2D navigation}
    \label{fig:2D-ablation}
\end{subfigure}
\hfill
\begin{subfigure}[t]{0.19\linewidth}
    \centering
    \includegraphics[width=0.95\linewidth]{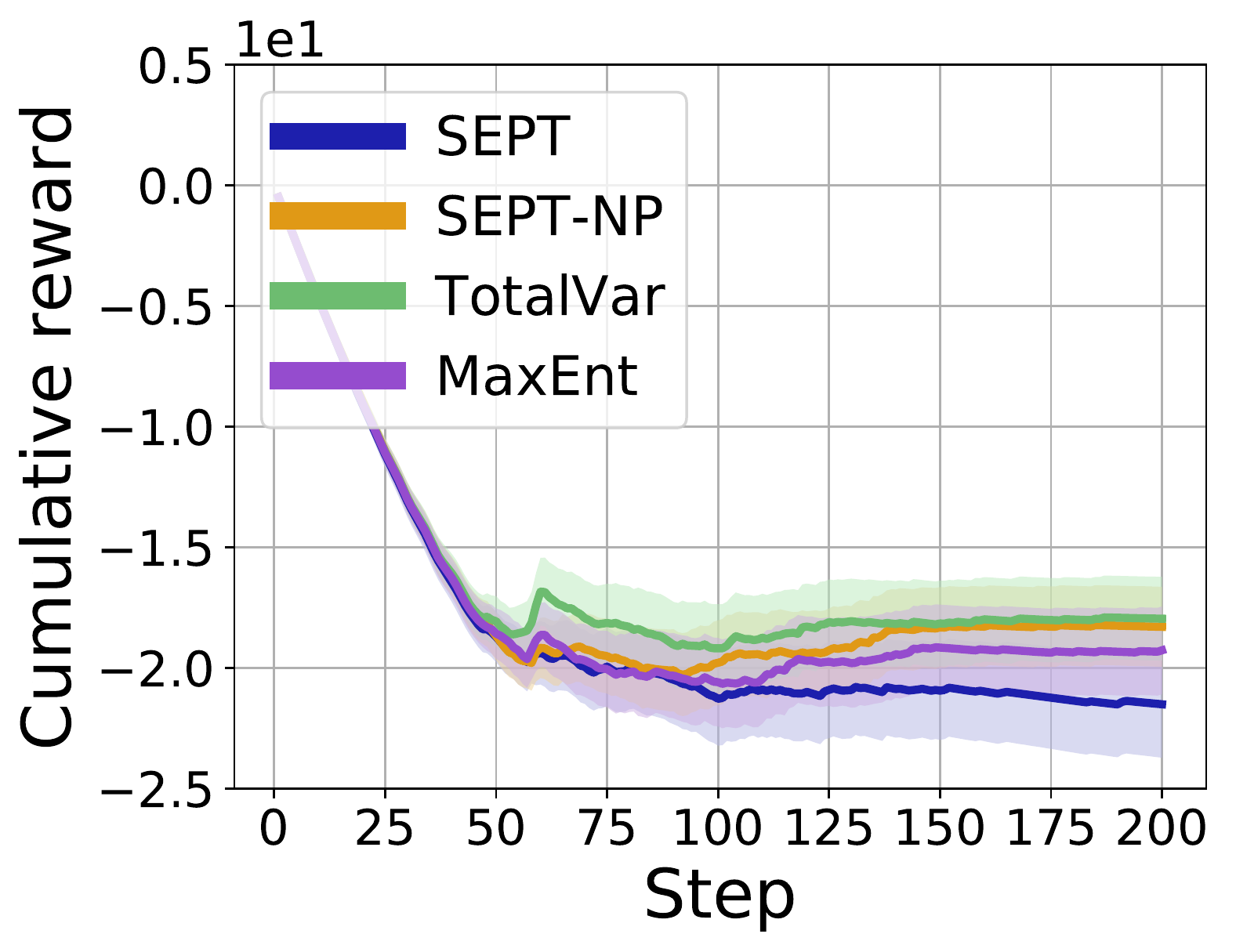}
    \caption{Acrobot}
    \label{fig:acrobot-ablation}
\end{subfigure}
\hfill
\begin{subfigure}[t]{0.19\linewidth}
    \centering
    \includegraphics[width=0.95\linewidth]{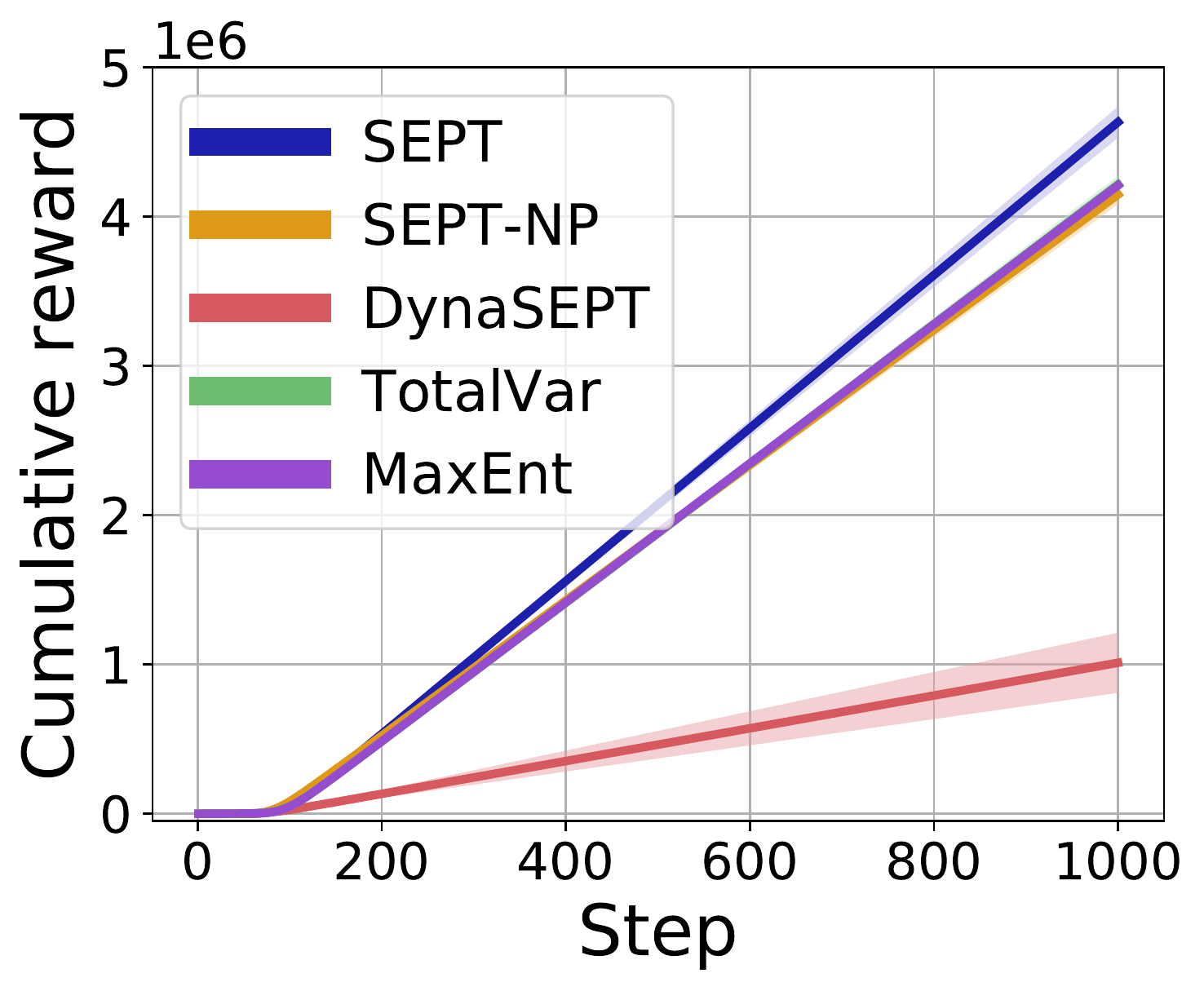}
    \caption{HIV}
    \label{fig:hiv-ablation}
\end{subfigure}
\caption{(a-e): Comparison against baselines. (a-b): Number of steps to solve 2D navigation and Acrobot in a single test episode; failure to solve is assigned a count of 50 in 2D nav and 200 in Acrobot. (c-e): Cumulative reward versus test episode step. BNN requires long computation time and showed low rewards on HIV, hence we report 3 seeds in \Cref{fig:hiv-with-bnn}. (f-j): Ablation results. DynaSEPT is out of range in (g), see \Cref{fig:acrobot-ablation-dyna}. Error bars show standard error of mean over all test instances over 20 training runs per method.}
\label{fig:main-comparisons}
\vspace{-0.5cm}
\end{figure}

\section{Results and discussion}
\label{sec:results}

2D navigation and Acrobot are solved upon attaining terminal reward of 1000 and 10, respectively.
SEPT outperforms all baselines in 2D navigation and takes significantly fewer number of steps to solve
(\Cref{fig:2D-steps,fig:2D}).
% SEPT attained the same maximum test performance as the Oracle, in nearly equal number of steps.
While a single instance of 2D navigation is easy for RL, handling multiple instances is highly non-trivial.
EPOpt-adv and Avg almost never solve the test instance---we set ``steps to solve'' to 50 for test episodes that were unsolved---because interpolating between instance-specific optimal policies in policy parameter space is not meaningful for any task instance.
MAML did not perform well despite having the advantage of being provided with rewards at test time, unlike SEPT. The gradient adaptation step was likely ineffective because the rewards are sparse and delayed.
BNN requires significantly more steps than SEPT, and it uses four orders of magnitude longer computation time (\Cref{table:test-times}), due to training a policy from scratch during the test episode.
Training times of all algorithms except BNN are in the same order of magnitude (\Cref{table:train-times}).
% \Cref{fig:2D,fig:acrobot,fig:hiv} show cumulative reward across the test episode.

In Acrobot and HIV, where dynamics are continuous in latent variables, interpolation within policy space can produce meaningful policies, so all baselines are feasible in principle.
SEPT is statistically significantly faster than BNN and Avg, is within error bars of MAML, while EPOpt-adv outperforms the rest by a small margin (\Cref{fig:acrobot-steps,fig:acrobot}).
\Cref{fig:percent} shows that SEPT is competitive in terms of percentage of solved instances.
As the true values of latent variables for Acrobot test instances were interpolated and extrapolated from the training values, this shows that SEPT is robust to out-of-training dynamics.
% High variance in Acrobot reward is due to different instances requiring drastically different number of steps to solve; however, \Cref{fig:acrobot-steps} shows that SEPT solved all instances faster than alternatives.
BNN requires more steps due to simultaneously learning and executing control during the test episode.
% On HIV, SEPT is within margin of error with EPOpt-adv, and outperformed other baselines (\Cref{fig:hiv}).
On HIV, SEPT reaches significantly higher cumulative rewards than all methods.
Oracle is within the margin of error of Avg. This may be due to insufficient examples of the high-dimensional ground truth hidden parameters.
Due to its long computational time, we run three seeds for BNN on HIV, shown in \Cref{fig:hiv-with-bnn}, and find it was unable to adapt within one test episode.

Comparing directly to reported results in DPT \citep{yao2018direct}, SEPT solves 2D Navigation at least 33\% (>10 steps) faster,
and the cumulative reward of SEPT (mean and standard error) are above DPT's mean cumulative reward in Acrobot (\Cref{table:acrobot-errorbar}).
% and solves Acrobot at least 20\% (>20 steps) faster.
Together, these results show that methods that explicitly distinguish different dynamics (e.g., SEPT and BNN) can significantly outperform methods that implicitly interpolate in policy parameter space (e.g., Avg and EPOpt-adv) in settings where $z$ has large discontinuous effect on dynamics, such as 2D navigation.
When dynamics are continuous in latent variables (e.g., Acrobot and HIV), interpolation-based methods fare better than BNN, which faces the difficulty of learning a model of the entire family of dynamics.
SEPT worked the best in the first case and is robust to the second case because it explicitly distinguishes dynamics and does not require learning a full transition model. Moreover, SEPT does not require rewards at test time allowing it be useful on a broader class of problems than optimization-based meta-learning approaches like MAML.
\Cref{app:results} contains training curves.

\begin{figure}[t]
\centering
\begin{subfigure}[t]{0.16\linewidth}
    \centering
    \includegraphics[width=1.0\linewidth]{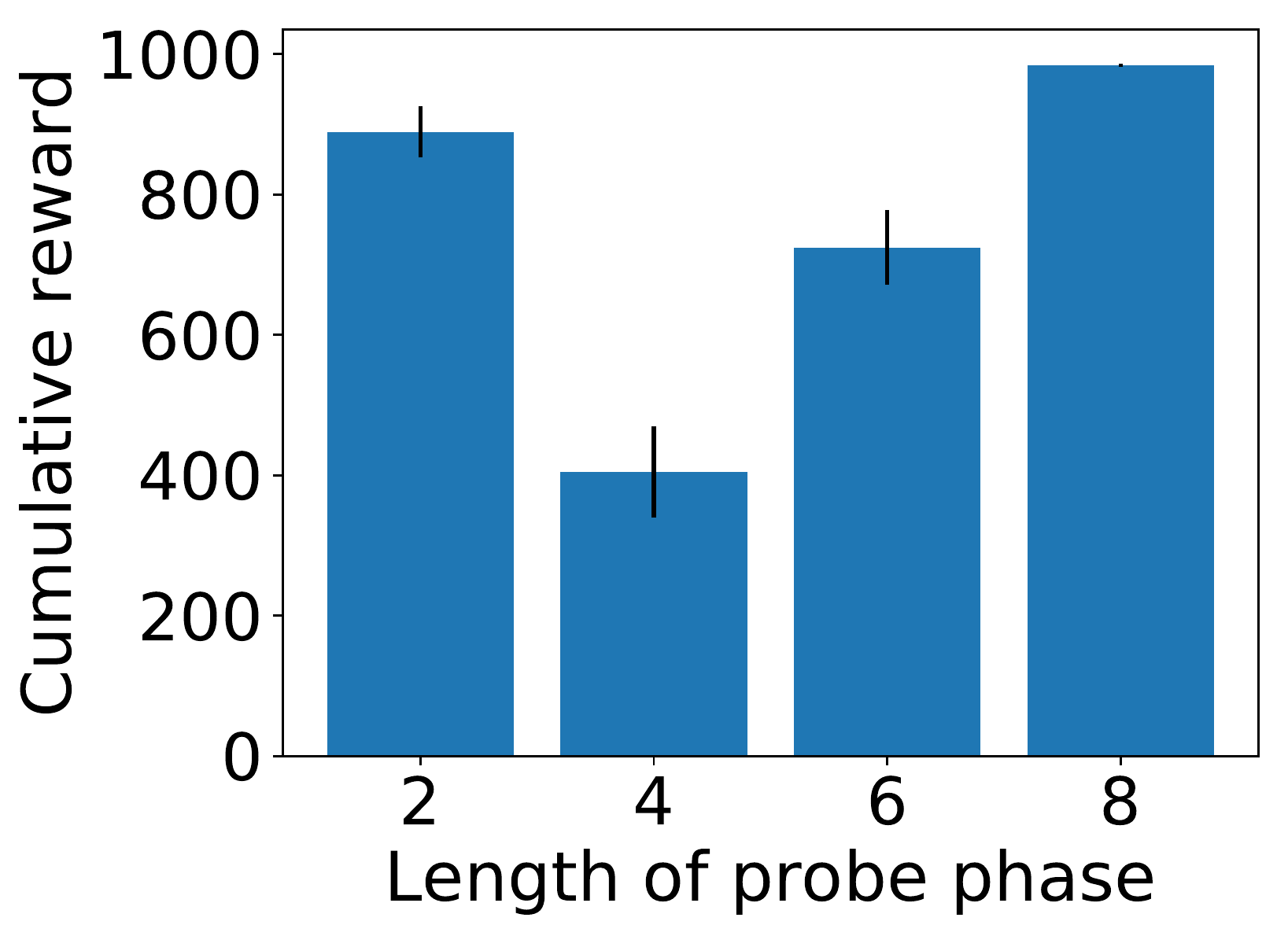}
    \caption{2D navigation}
    \label{fig:2D-length}
\end{subfigure}
\hfill
\begin{subfigure}[t]{0.16\linewidth}
    \centering
    \includegraphics[width=1.0\linewidth]{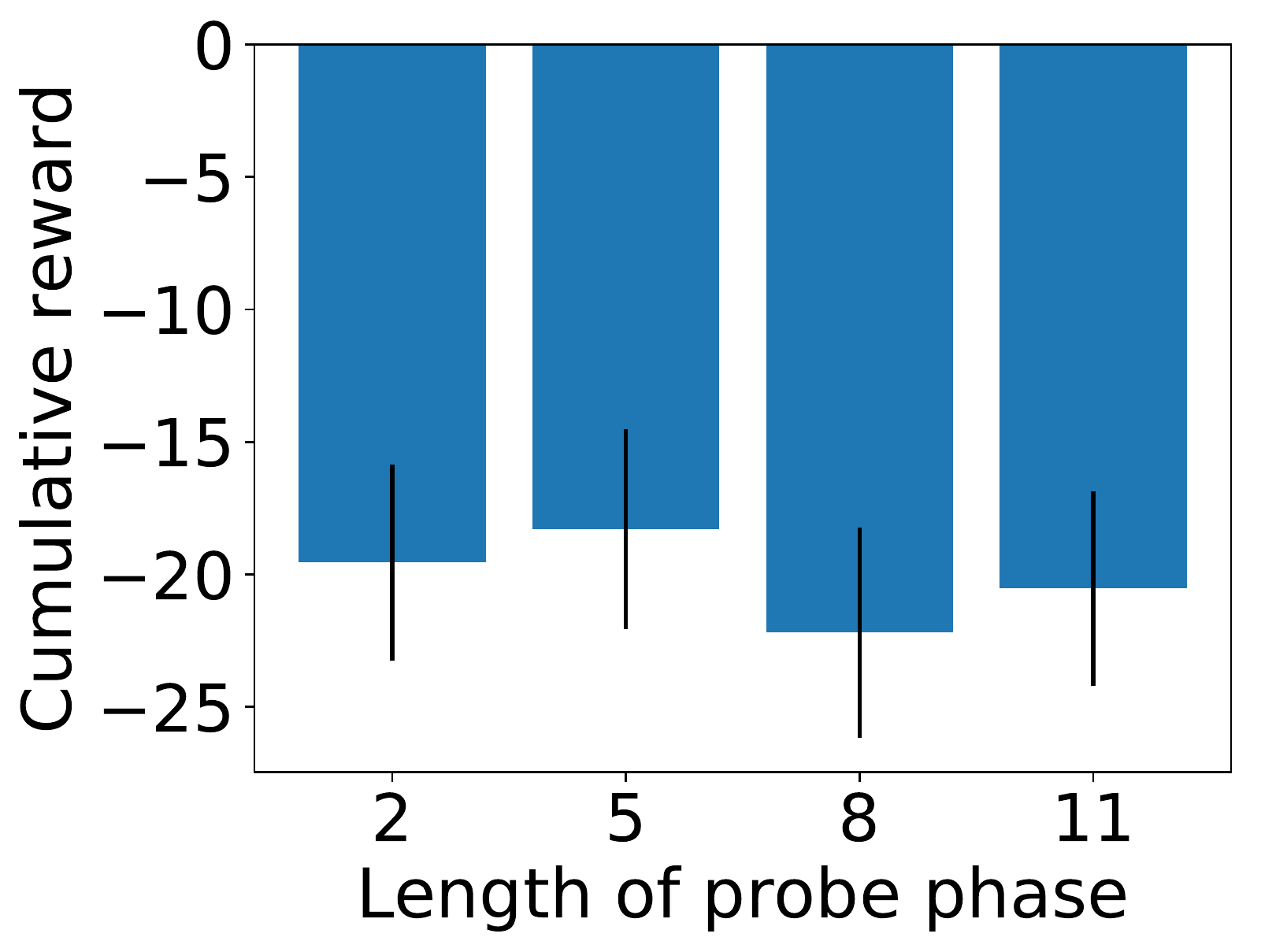}
    \caption{Acrobot}
    \label{fig:acrobot-length}
\end{subfigure}
\hfill
\begin{subfigure}[t]{0.16\linewidth}
    \centering
    \includegraphics[width=0.9\linewidth]{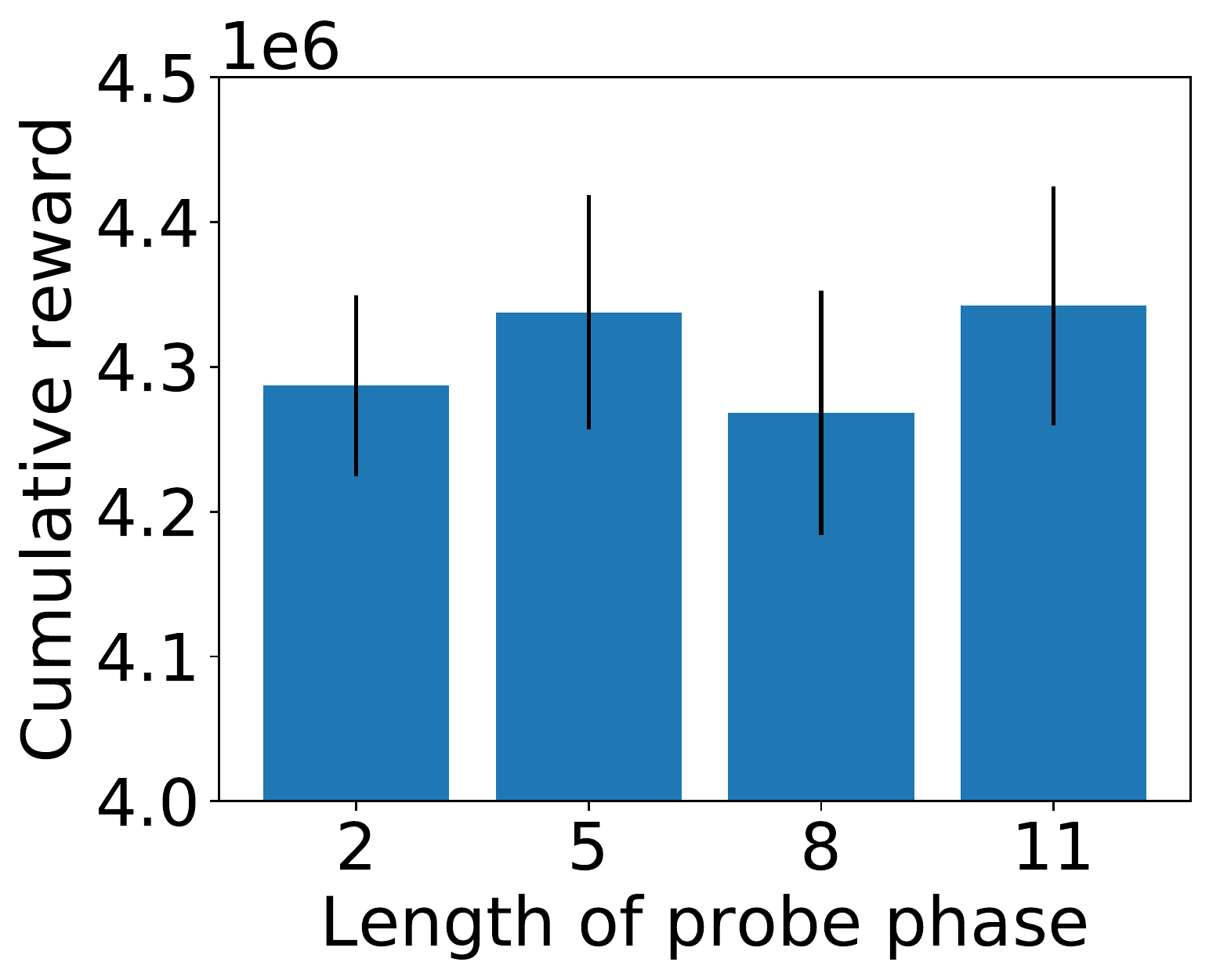}
    \caption{HIV}
    \label{fig:hiv-length}
\end{subfigure}
\begin{subfigure}[t]{0.16\linewidth}
    \centering
    \includegraphics[width=1.0\linewidth]{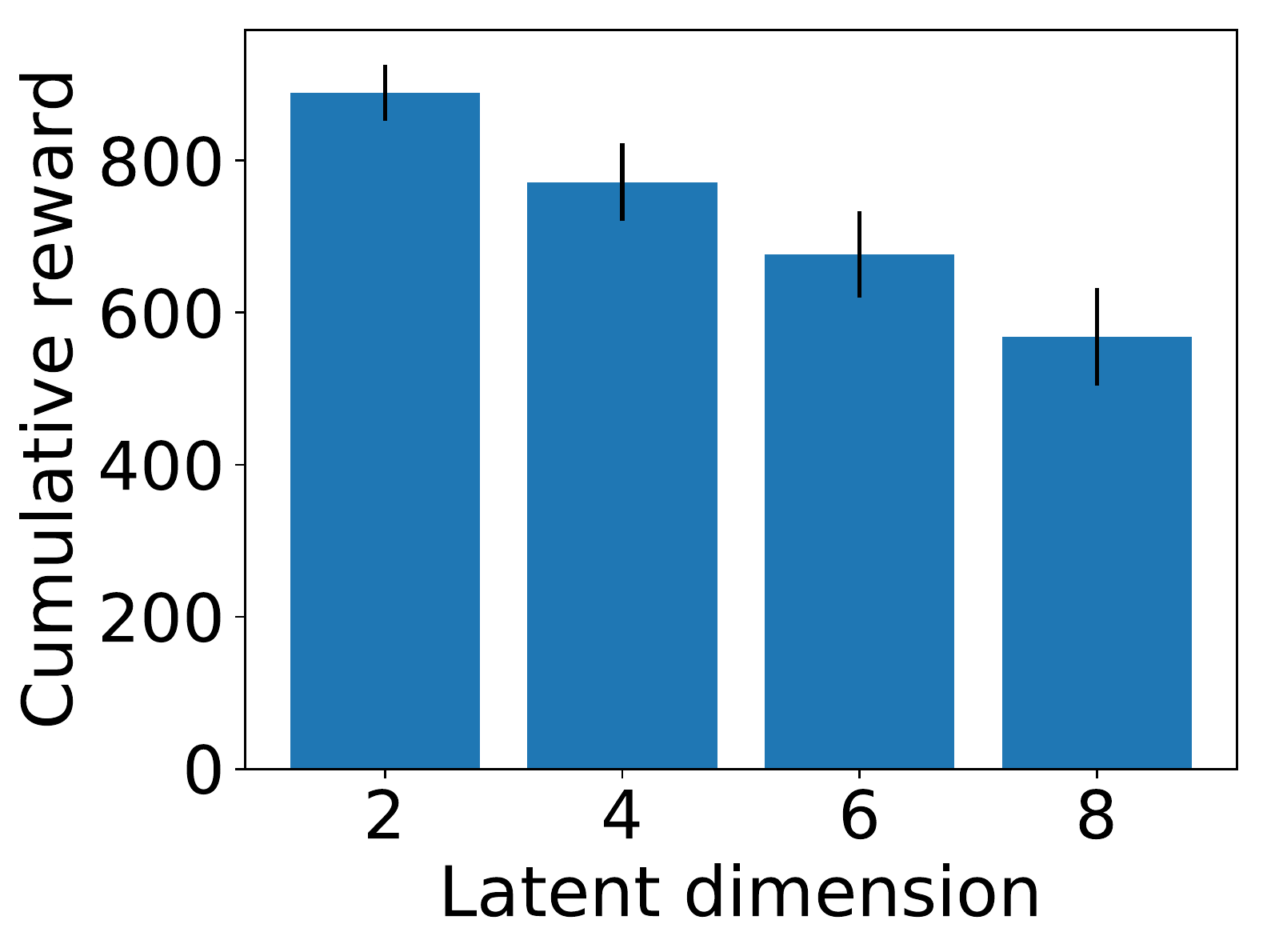}
    \caption{2D navigation}
    \label{fig:2D-r-latent}
\end{subfigure}
\hfill
\begin{subfigure}[t]{0.16\linewidth}
    \centering
    \includegraphics[width=1.0\linewidth]{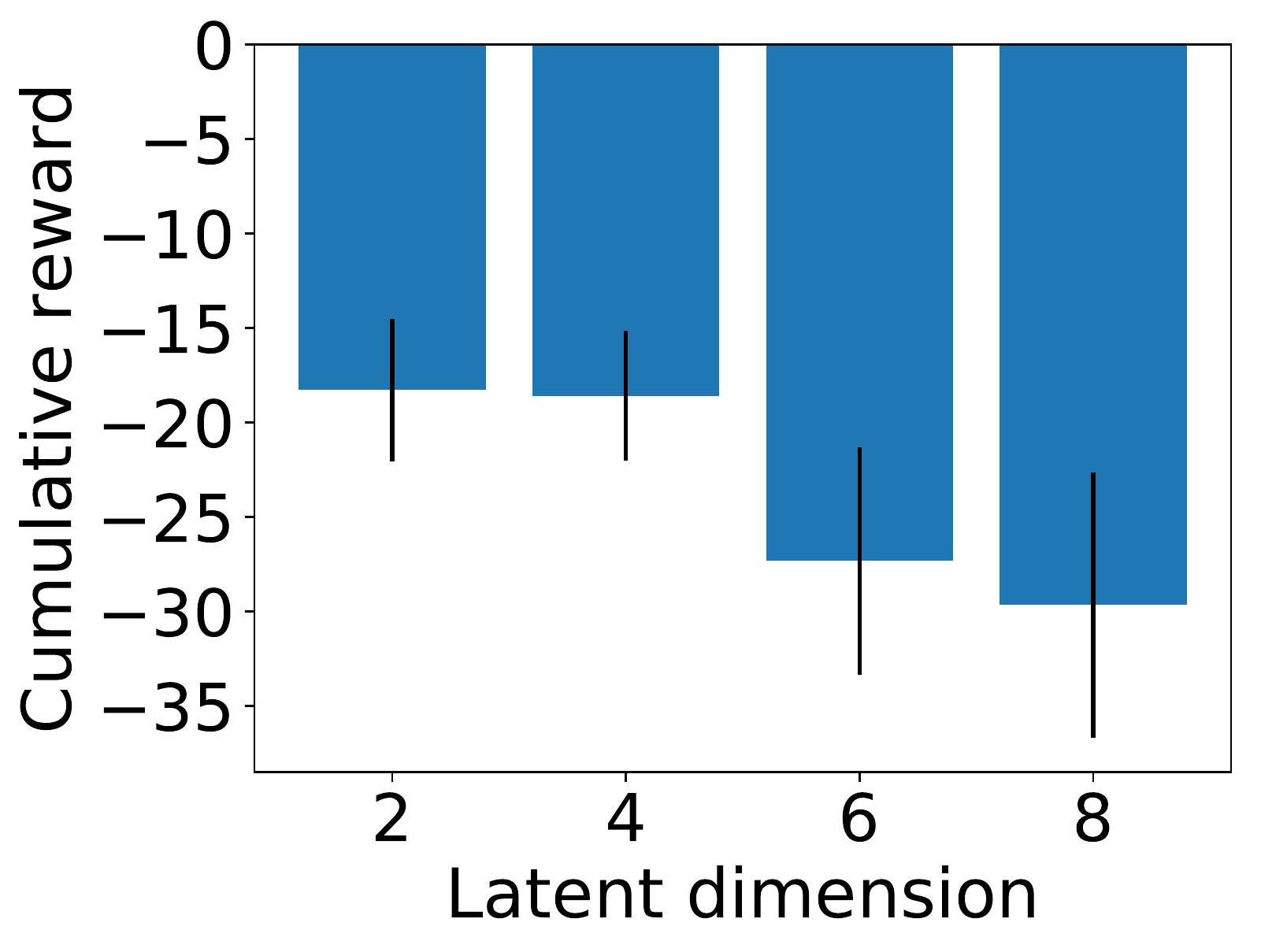}
    \caption{Acrobot}
    \label{fig:acrobot-r-latent}
\end{subfigure}
\hfill
\begin{subfigure}[t]{0.16\linewidth}
    \centering
    \includegraphics[width=0.9\linewidth]{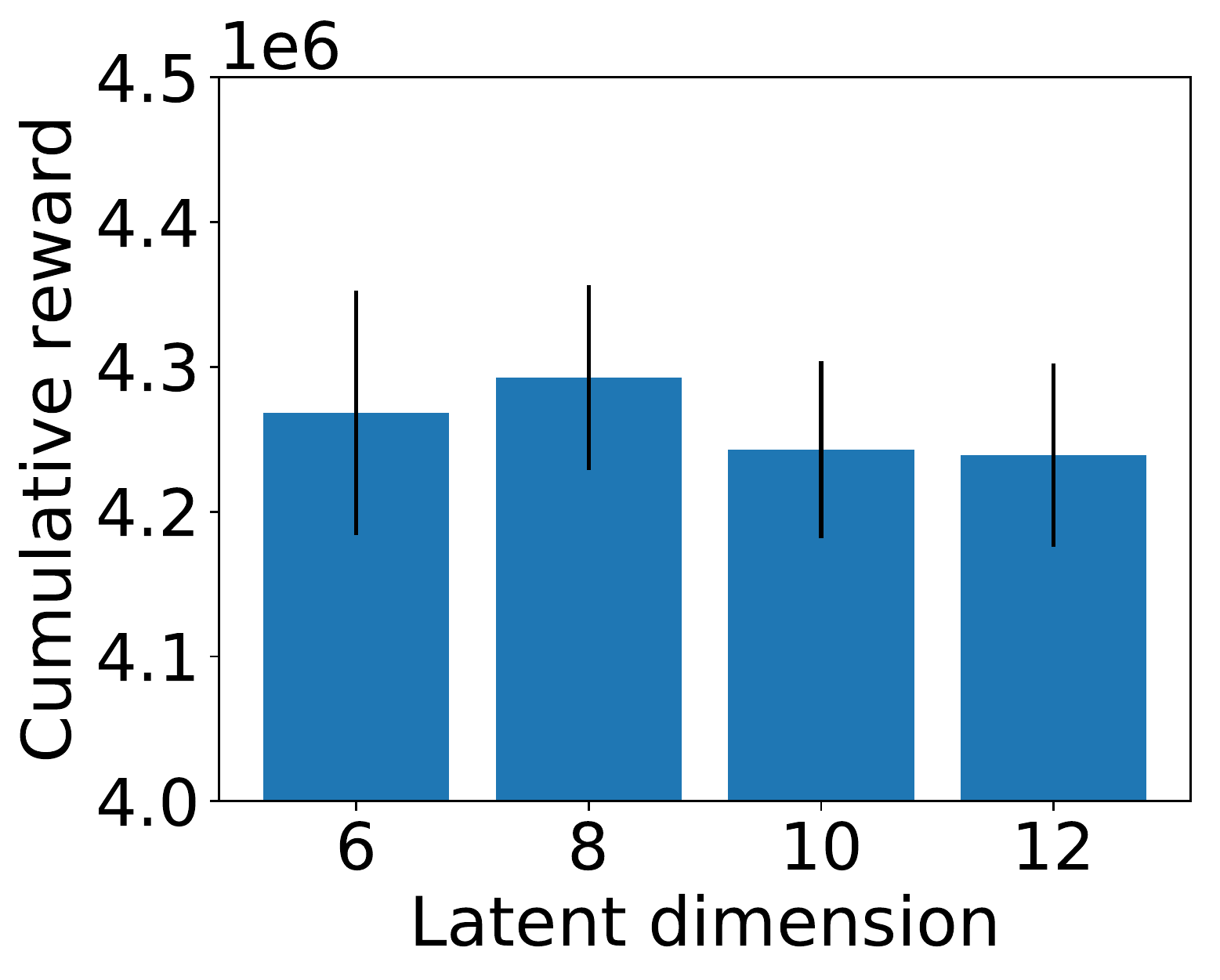}
    \caption{HIV}
    \label{fig:hiv-r-latent}
\end{subfigure}
\caption{Cumulative reward on test episode for different $T_p$ (a-c) and different $\dim(z)$ (d-f). $8M$ independent test instances for each hyperparameter setting.}
\label{fig:hyperparam}
% \vspace{-0.3cm}
\end{figure}

\textbf{Ablation results.}
Comparing to SEPT-NP, \Cref{fig:2D-ablation-steps,fig:acrobot-ablation-steps,fig:hiv-ablation} show that the probe phase is necessary to solve 2D navigation quickly, while giving similar performance in Acrobot and significant improvement in HIV.
SEPT significantly outperformed TotalVar in 2D navigation and HIV, while TotalVar gives slight improvement in Acrobot, showing that directly using VAE performance as the reward for probing in certain environments can be more effective than a reward that deliberately encourages perturbation of state dimensions.
The clear advantage of SEPT over MaxEnt in 2D navigation and HIV supports our hypothesis in \Cref{subsec:probe} that the variational lowerbound, rather than its negation in the maximum entropy viewpoint, should be used as the probe reward, while performance was not significantly differentiated in Acrobot.
% Comparing to SEPT-NP, \Cref{fig:2D-ablation-steps,fig:acrobot-ablation-steps,fig:hiv-ablation} show that the probe phase is necessary to solve 2D navigation and Acrobot quickly, while giving marginal improvement in HIV.
% SEPT matched the performance of TotalVar in HIV and outperformed in 2D navigation and Acrobot, showing that directly using VAE performance as reward for probing is more effective than a reward that deliberately encourages perturbation of state dimensions.
% The clear advantage of SEPT over MaxEnt in all three domains supports our hypothesis in \Cref{subsec:probe} that the variational lowerbound, rather than its negation in the maximum entropy viewpoint, should be used as the probe reward.
SEPT outperforms DynaSEPT on all problems where dynamics are stationary during each instance.
On the other hand, DynaSEPT is the better choice in a non-stationary variant of 2D navigation where the dynamics ``switch'' abruptly at $t=10$ (\Cref{fig:2D-switch}).

\textbf{Robustness.}
\Cref{fig:hyperparam} shows that SEPT is robust to varying the probe length $T_p$ and $\dim(z)$.
Even with certain suboptimal probe length and $\dim(z)$, it can outperform all baselines on 2D navigation in both steps-to-solve and final reward; it is within error bars of all baselines on Acrobot based on final cumulative reward;
and final cumulative reward exceeds that of baselines in HIV.
% and it meets the performance of EPOpt-adv while outperforming other baselines on HIV.
Increasing $T_p$ means foregoing valuable steps of the control policy and increasing difficulty of trajectory reconstruction for the VAE in high dimensional state spaces;
$T_p$ is a hyper-parameter that should be validated for each application.
% hence a longer probing phase is not always better.
\Cref{app:latent-dynamics} shows the effect of $\beta$ on latent variable encodings.

\section{Conclusion and future directions}

We propose a general algorithm for single episode transfer among MDPs with different stationary dynamics, which is a challenging goal with real-world significance that deserves increased effort from the transfer learning and RL community.
Our method, Single Episode Policy Transfer (SEPT), trains a probe policy and an inference model to discover a latent representation of dynamics using very few initial steps in a single test episode, such that a universal policy can execute optimal control without access to rewards at test time.
Strong performance versus baselines in domains involving both continuous and discontinuous dependence of dynamics on latent variables show the promise of SEPT for problems where different dynamics can be distinguished via a short probing phase.

The dedicated probing phase may be improved by other objectives, in addition to performance of the inference model, to mitigate the risk and opportunity cost of probing.
An open challenge is single episode transfer in domains where differences in dynamics of different instances are not detectable early during an episode, or where latent variables are fixed but dynamics are nonstationary.
Further research on dynamic probing and control, as sketched in DynaSEPT, is one path toward addressing this challenge.
% While SEPT generates short probe trajectories to infer differences in dynamics, 
% one can consider how best to use available long trajectory data with machine teaching methods when differences in dynamics are not detectable early during an episode.
Our work is one step along a broader avenue of research on general transfer learning in RL equipped with the realistic constraint of a single episode for adaptation and evaluation.

\subsubsection*{Acknowledgments}
This work was performed under the auspices of the U.S. Department of Energy by Lawrence Livermore National Laboratory under contract DE-AC52-07NA27344. Lawrence Livermore National Security, LLC. LLNL-JRNL-791194.

\bibliographystyle{natbib}
\bibliography{citation}

\begin{thebibliography}{}

\bibitem[Adams {\em et~al.}(2004)Adams, Banks, Kwon, and
  Tran]{adams2004dynamic}
Adams, B.~M., Banks, H.~T., Kwon, H.-D., and Tran, H.~T. (2004).
\newblock Dynamic multidrug therapies for hiv: Optimal and sti control
  approaches.
\newblock {\em Mathematical biosciences and engineering\/}, {\bf 1}(2),
  223--241.

\bibitem[Arnekvist {\em et~al.}(2019)Arnekvist, Kragic, and
  Stork]{arnekvist2019vpe}
Arnekvist, I., Kragic, D., and Stork, J.~A. (2019).
\newblock Vpe: Variational policy embedding for transfer reinforcement
  learning.
\newblock In {\em 2019 International Conference on Robotics and Automation
  (ICRA)\/}, pages 36--42. IEEE.

\bibitem[Bordbar {\em et~al.}(2015)Bordbar, McCloskey, Zielinski, Sonnenschein,
  Jamshidi, and Palsson]{bordbar2015personalized}
Bordbar, A., McCloskey, D., Zielinski, D.~C., Sonnenschein, N., Jamshidi, N.,
  and Palsson, B.~O. (2015).
\newblock Personalized whole-cell kinetic models of metabolism for discovery in
  genomics and pharmacodynamics.
\newblock {\em Cell systems\/}, {\bf 1}(4), 283--292.

\bibitem[Caruana(1997)Caruana]{caruana1997multitask}
Caruana, R. (1997).
\newblock Multitask learning.
\newblock {\em Machine learning\/}, {\bf 28}(1), 41--75.

\bibitem[Co-Reyes {\em et~al.}(2018)Co-Reyes, Liu, Gupta, Eysenbach, Abbeel,
  and Levine]{co2018self}
Co-Reyes, J.~D., Liu, Y., Gupta, A., Eysenbach, B., Abbeel, P., and Levine, S.
  (2018).
\newblock Self-consistent trajectory autoencoder: Hierarchical reinforcement
  learning with trajectory embeddings.
\newblock In {\em Proceedings of the 35th International Conference on Machine
  Learning\/}, pages 1009--1018.

\bibitem[Cockrell and An(2018)Cockrell and An]{Cockrell2018responders}
Cockrell, R.~C. and An, G. (2018).
\newblock Examining the controllability of sepsis using genetic algorithms on
  an agent-based model of systemic inflammation.
\newblock {\em PLOS Computational Biology\/}, {\bf 14}(2), 1--17.

\bibitem[Da~Silva {\em et~al.}(2012)Da~Silva, Konidaris, and
  Barto]{da2012learning}
Da~Silva, B.~C., Konidaris, G., and Barto, A.~G. (2012).
\newblock Learning parameterized skills.
\newblock In {\em Proceedings of the 29th International Coference on
  International Conference on Machine Learning\/}, pages 1443--1450. Omnipress.

\bibitem[Doshi-Velez and Konidaris(2016)Doshi-Velez and
  Konidaris]{doshi2016hidden}
Doshi-Velez, F. and Konidaris, G. (2016).
\newblock Hidden parameter markov decision processes: A semiparametric
  regression approach for discovering latent task parametrizations.
\newblock In {\em IJCAI: proceedings of the conference\/}, volume 2016, page
  1432. NIH Public Access.

\bibitem[Finn {\em et~al.}(2017)Finn, Abbeel, and Levine]{finn2017model}
Finn, C., Abbeel, P., and Levine, S. (2017).
\newblock Model-agnostic meta-learning for fast adaptation of deep networks.
\newblock In {\em Proceedings of the 34th International Conference on Machine
  Learning-Volume 70\/}, pages 1126--1135.

\bibitem[Higgins {\em et~al.}(2017)Higgins, Matthey, Pal, Burgess, Glorot,
  Botvinick, Mohamed, and Lerchner]{higgins2017beta}
Higgins, I., Matthey, L., Pal, A., Burgess, C., Glorot, X., Botvinick, M.,
  Mohamed, S., and Lerchner, A. (2017).
\newblock beta-vae: Learning basic visual concepts with a constrained
  variational framework.
\newblock In {\em International Conference on Learning Representations\/},
  volume~3.

\bibitem[Hodson(2016)Hodson]{hodson2016precision}
Hodson, R. (2016).
\newblock Precision medicine.
\newblock {\em Nature\/}, {\bf 537}(7619), S49.

\bibitem[Kaelbling {\em et~al.}(1998)Kaelbling, Littman, and
  Cassandra]{kaelbling1998planning}
Kaelbling, L.~P., Littman, M.~L., and Cassandra, A.~R. (1998).
\newblock Planning and acting in partially observable stochastic domains.
\newblock {\em Artificial intelligence\/}, {\bf 101}(1-2), 99--134.

\bibitem[Kalm{\'a}r {\em et~al.}(1998)Kalm{\'a}r, Szepesv{\'a}ri, and
  L{\H{o}}rincz]{kalmar1998module}
Kalm{\'a}r, Z., Szepesv{\'a}ri, C., and L{\H{o}}rincz, A. (1998).
\newblock Module-based reinforcement learning: Experiments with a real robot.
\newblock {\em Autonomous Robots\/}, {\bf 5}(3-4), 273--295.

\bibitem[Kastrin {\em et~al.}(2018)Kastrin, Ferk, and
  Lesko{\v{s}}ek]{kastrin2018predicting}
Kastrin, A., Ferk, P., and Lesko{\v{s}}ek, B. (2018).
\newblock Predicting potential drug-drug interactions on topological and
  semantic similarity features using statistical learning.
\newblock {\em PloS one\/}, {\bf 13}(5), e0196865.

\bibitem[Killian {\em et~al.}(2017)Killian, Daulton, Konidaris, and
  Doshi-Velez]{killian2017robust}
Killian, T.~W., Daulton, S., Konidaris, G., and Doshi-Velez, F. (2017).
\newblock Robust and efficient transfer learning with hidden parameter markov
  decision processes.
\newblock In {\em Advances in Neural Information Processing Systems\/}, pages
  6250--6261.

\bibitem[Kingma and Welling(2014)Kingma and Welling]{kingma2013auto}
Kingma, D.~P. and Welling, M. (2014).
\newblock Auto-encoding variational bayes.
\newblock In {\em International Conference on Learning Representations\/}.

\bibitem[Konidaris and Doshi-Velez(2014)Konidaris and
  Doshi-Velez]{konidaris2014hidden}
Konidaris, G. and Doshi-Velez, F. (2014).
\newblock Hidden parameter markov decision processes: an emerging paradigm for
  modeling families of related tasks.
\newblock In {\em the AAAI Fall Symposium on Knowledge, Skill, and Behavior
  Transfer in Autonomous Robots\/}.

\bibitem[Mnih {\em et~al.}(2015)Mnih, Kavukcuoglu, Silver, Rusu, Veness,
  Bellemare, Graves, Riedmiller, Fidjeland, Ostrovski, {\em
  et~al.}]{mnih2015human}
Mnih, V., Kavukcuoglu, K., Silver, D., Rusu, A.~A., Veness, J., Bellemare,
  M.~G., Graves, A., Riedmiller, M., Fidjeland, A.~K., Ostrovski, G., {\em
  et~al.} (2015).
\newblock Human-level control through deep reinforcement learning.
\newblock {\em Nature\/}, {\bf 518}(7540), 529.

\bibitem[Paul {\em et~al.}(2019)Paul, Osborne, and
  Whiteson]{paul2019fingerprint}
Paul, S., Osborne, M.~A., and Whiteson, S. (2019).
\newblock Fingerprint policy optimisation for robust reinforcement learning.
\newblock In {\em International Conference on Machine Learning\/}, pages
  5082--5091.

\bibitem[Perez {\em et~al.}(2018)Perez, Such, and
  Karaletsos]{perez2018efficient}
Perez, C.~F., Such, F.~P., and Karaletsos, T. (2018).
\newblock Efficient transfer learning and online adaptation with latent
  variable models for continuous control.
\newblock {\em arXiv preprint arXiv:1812.03399\/}.

\bibitem[Petersen {\em et~al.}(2019)Petersen, Yang, Grathwohl, Cockrell,
  Santiago, An, and Faissol]{petersen2019deep}
Petersen, B.~K., Yang, J., Grathwohl, W.~S., Cockrell, C., Santiago, C., An,
  G., and Faissol, D.~M. (2019).
\newblock Deep reinforcement learning and simulation as a path toward precision
  medicine.
\newblock {\em Journal of Computational Biology\/}.

\bibitem[Rajeswaran {\em et~al.}(2017)Rajeswaran, Ghotra, Ravindran, and
  Levine]{rajeswaran2016epopt}
Rajeswaran, A., Ghotra, S., Ravindran, B., and Levine, S. (2017).
\newblock Epopt: Learning robust neural network policies using model ensembles.
\newblock In {\em International Conference on Learning Representations\/}.

\bibitem[Rakelly {\em et~al.}(2019)Rakelly, Zhou, Finn, Levine, and
  Quillen]{rakelly2019efficient}
Rakelly, K., Zhou, A., Finn, C., Levine, S., and Quillen, D. (2019).
\newblock Efficient off-policy meta-reinforcement learning via probabilistic
  context variables.
\newblock In {\em International Conference on Machine Learning\/}, pages
  5331--5340.

\bibitem[Rusu {\em et~al.}(2019)Rusu, Rao, Sygnowski, Vinyals, Pascanu,
  Osindero, and Hadsell]{rusu2018meta}
Rusu, A.~A., Rao, D., Sygnowski, J., Vinyals, O., Pascanu, R., Osindero, S.,
  and Hadsell, R. (2019).
\newblock Meta-learning with latent embedding optimization.
\newblock In {\em International Conference Learning Representations (ICLR)\/}.

\bibitem[Schaul {\em et~al.}(2015)Schaul, Horgan, Gregor, and
  Silver]{schaul2015universal}
Schaul, T., Horgan, D., Gregor, K., and Silver, D. (2015).
\newblock Universal value function approximators.
\newblock In {\em International Conference on Machine Learning\/}, pages
  1312--1320.

\bibitem[Schaul {\em et~al.}(2016)Schaul, Quan, Antonoglou, and
  Silver]{schaul2015prioritized}
Schaul, T., Quan, J., Antonoglou, I., and Silver, D. (2016).
\newblock Prioritized experience replay.
\newblock In {\em International Conference Learning Representations (ICLR)\/},
  volume 2016.

\bibitem[Sutton {\em et~al.}(1998)Sutton, Barto, {\em
  et~al.}]{sutton1998reinforcement}
Sutton, R.~S., Barto, A.~G., {\em et~al.} (1998).
\newblock {\em Reinforcement learning: An introduction\/}.
\newblock MIT press.

\bibitem[Szita {\em et~al.}(2002)Szita, Tak{\'a}cs, and
  L{\"o}rincz]{szita2002varepsilon}
Szita, I., Tak{\'a}cs, B., and L{\"o}rincz, A. (2002).
\newblock $\varepsilon$-mdps: Learning in varying environments.
\newblock {\em Journal of Machine Learning Research\/}, {\bf 3}(Aug), 145--174.

\bibitem[Taylor and Stone(2009)Taylor and Stone]{taylor2009transfer}
Taylor, M.~E. and Stone, P. (2009).
\newblock Transfer learning for reinforcement learning domains: A survey.
\newblock {\em Journal of Machine Learning Research\/}, {\bf 10}(Jul),
  1633--1685.

\bibitem[Tirinzoni {\em et~al.}(2018)Tirinzoni, Sanchez, and
  Restelli]{tirinzoni2018transfer}
Tirinzoni, A., Sanchez, R.~R., and Restelli, M. (2018).
\newblock Transfer of value functions via variational methods.
\newblock In {\em Advances in Neural Information Processing Systems\/}, pages
  6179--6189.

\bibitem[Van~Hasselt {\em et~al.}(2016)Van~Hasselt, Guez, and
  Silver]{van2016deep}
Van~Hasselt, H., Guez, A., and Silver, D. (2016).
\newblock Deep reinforcement learning with double q-learning.
\newblock In {\em Thirtieth AAAI Conference on Artificial Intelligence\/}.

\bibitem[West {\em et~al.}(2018)West, You, Brown, Newton, and
  Anderson]{West476507}
West, J., You, L., Brown, J., Newton, P.~K., and Anderson, A. R.~A. (2018).
\newblock Towards multi-drug adaptive therapy.
\newblock {\em bioRxiv\/}.

\bibitem[Whirl-Carrillo {\em et~al.}(2012)Whirl-Carrillo, McDonagh, Hebert,
  Gong, Sangkuhl, Thorn, Altman, and Klein]{whirl2012pharmacogenomics}
Whirl-Carrillo, M., McDonagh, E.~M., Hebert, J., Gong, L., Sangkuhl, K., Thorn,
  C., Altman, R.~B., and Klein, T.~E. (2012).
\newblock Pharmacogenomics knowledge for personalized medicine.
\newblock {\em Clinical Pharmacology \& Therapeutics\/}, {\bf 92}(4), 414--417.

\bibitem[Williams(1992)Williams]{williams1992simple}
Williams, R.~J. (1992).
\newblock Simple statistical gradient-following algorithms for connectionist
  reinforcement learning.
\newblock {\em Machine learning\/}, {\bf 8}(3-4), 229--256.

\bibitem[Xu {\em et~al.}(2018)Xu, Liu, Zhao, Xu, and Peng]{xu2018learning}
Xu, T., Liu, Q., Zhao, L., Xu, W., and Peng, J. (2018).
\newblock Learning to explore with meta-policy gradient.
\newblock In {\em Proceedings of the 35th International Conference on Machine
  Learning\/}, pages 5463--5472.

\bibitem[Yao {\em et~al.}(2018)Yao, Killian, Konidaris, and
  Doshi-Velez]{yao2018direct}
Yao, J., Killian, T., Konidaris, G., and Doshi-Velez, F. (2018).
\newblock Direct policy transfer via hidden parameter markov decision
  processes.
\newblock In {\em LLARLA Workshop, FAIM\/}, volume 2018.

\bibitem[Yu {\em et~al.}(2017)Yu, Tan, Liu, and Turk]{Yu-RSS-17}
Yu, W., Tan, J., Liu, C.~K., and Turk, G. (2017).
\newblock Preparing for the unknown: Learning a universal policy with online
  system identification.
\newblock In {\em Proceedings of Robotics: Science and Systems\/}, Cambridge,
  Massachusetts.

\bibitem[Zhang {\em et~al.}(2018)Zhang, Satija, and
  Pineau]{zhang2018decoupling}
Zhang, A., Satija, H., and Pineau, J. (2018).
\newblock Decoupling dynamics and reward for transfer learning.
\newblock {\em arXiv preprint arXiv:1804.10689\/}.

\bibitem[Zhang {\em et~al.}(2017)Zhang, Cunningham, Brown, and
  Gatenby]{zhang2017integrating}
Zhang, J., Cunningham, J.~J., Brown, J.~S., and Gatenby, R.~A. (2017).
\newblock Integrating evolutionary dynamics into treatment of metastatic
  castrate-resistant prostate cancer.
\newblock {\em Nature communications\/}, {\bf 8}(1), 1816.

\bibitem[Zhu {\em et~al.}(2018)Zhu, Singla, Zilles, and
  Rafferty]{zhu2018overview}
Zhu, X., Singla, A., Zilles, S., and Rafferty, A.~N. (2018).
\newblock An overview of machine teaching.
\newblock {\em arXiv preprint arXiv:1801.05927\/}.

\bibitem[Zintgraf {\em et~al.}(2019)Zintgraf, Shiarlis, Kurin, Hofmann, and
  Whiteson]{zintgraf2018fast}
Zintgraf, L.~M., Shiarlis, K., Kurin, V., Hofmann, K., and Whiteson, S. (2019).
\newblock Fast context adaptation via meta-learning.
\newblock In {\em International Conference on Machine Learning (ICML)\/},
  volume 2019.

\end{thebibliography}

\newpage
\appendix

\section{Derivations}

\label{app:proof-exploration}
\propexploration*
\begin{proof}
Assuming regularity, the gradient of the entropy is
\begin{align*}
    \nabla_{\varphi} H(p_{\varphi}(\tau)) &= - \nabla_{\varphi} \int p_{\varphi}(\tau) \log p_{\varphi}(\tau) d\tau \\
    &= - \int \nabla_{\varphi} p_{\varphi}(\tau) d\tau - \int \bigl( \nabla_{\varphi} p_{\varphi}(\tau) \bigr) \log p_{\varphi}(\tau) d\tau \\
    &= -\nabla_{\varphi} \int p_{\varphi}(\tau) d\tau - \int p_{\varphi}(\tau) \bigl( \nabla_{\varphi} \log p_{\varphi}(\tau) \bigr) \log p_{\varphi}(\tau) d\tau \\
    &= \Ebb_{p_{\varphi}(\tau)} \bigl[ \bigl( \nabla_{\varphi} \log p_{\varphi}(\tau) \bigr) \bigl( - \log p_{\varphi}(\tau) \bigr) \bigr]
\end{align*}
For trajectory $\tau := (s_0,a_0,s_1,\dotsc,s_{t})$ generated by the probe policy $\pi_{\varphi}$:
\begin{align*}
    p_{\varphi}(\tau) = p(s_0)\prod_{i=0}^{t-1} p(s_{i+1}|s_i,a_i) \pi_{\varphi}(a_i|s_i)
\end{align*}
Then
\begin{align*}
    \nabla_{\varphi} \log p_{\varphi}(\tau) &= \nabla_{\varphi} \Bigl( \log p(s_0) + \sum_{i=0}^{t-1} \log p(s_{i+1}|s_i,a_i) + \sum_{i=0}^{t-1} \log \pi_{\varphi}(a_i|s_i) \Bigr)
\end{align*}
Since $p(s_0)$ and $p(s_{i+1}|s_i,a_i)$ do not depend on $\varphi$, we get
\begin{align*}
    \nabla_{\varphi} \log p_{\varphi}(\tau) = \nabla_{\varphi} \sum_{i=0}^{t-1} \log \pi_{\varphi}(a_i|s_i)
\end{align*}
Substituting this into the gradient of the entropy gives \eqref{eq:exploration-gradient}.
\end{proof}

\section{Testing phase of SEPT}
\label{app:test}

\begin{algorithm}[ht]
\caption{Single Episode Policy Transfer: testing phase}
\label{alg:sept-test}
\begin{algorithmic}[1]
\Procedure{SEPT-test}{}
\State Restore trained decoder $\psi$, encoder $\phi$, probe policy $\varphi$, and control policy $\theta$
\State Run probe policy $\pi_{\varphi}$ for $T_p$ time steps and record trajectory $\tau_p$
\State Use $\tau_p$ with decoder $q_{\phi}(z|\tau)$ to estimate $\zhat$
\State Use $\zhat$ with control policy $\pi_{\theta}(a|s,z)$ for the remaining duration of the test episode
\EndProcedure
\end{algorithmic}
\end{algorithm}

\section{DynaSEPT}
\label{app:dynasept}

In our problem formulation, it is not necessary to compute $\zhat$ at every step of the test episode, as each instance is a stationary MDP and change of instances is known.
However, removing the common assumption of stationarity leads to time-dependent transition functions $\Tcal_z(s'|s,a)$, which introduces problematic cases.
For example, a length $T_p$ probing phase would fail if $z$ leads to a switch in dynamics at time $t > T_p$,
such as when poorly understood drug-drug interactions lead to abrupt changes in dynamics during co-medication therapies \citep{kastrin2018predicting}.
Here we describe an alternative general algorithm for non-stationary dynamics, which we call DynaSEPT.
We train a single policy $\pi_{\theta}(a|s,z,\eta)$ that dynamically decides whether to probe for better inference or act to maximize the MDP reward $R_{\text{env}}$, based on a scalar-valued function $\eta \colon \Rbb \rightarrow [0,1]$ representing the degree of uncertainty in posterior inference, which is updated at every time step.
The total reward is $R_\textrm{tot}(\tau) := \eta R_p(\tau) + (1-\eta) R_\textrm{env}(\tau_f)$,
where $\tau$ is a short sliding-window trajectory of length $T_p$, and $\tau_f$ is the final state of $\tau$.
The history-dependent term $R_p(\tau)$ is equivalent to a delayed reward given for executing a sequence of probe actions.
Following the same reasoning for SEPT, one choice for $R_p(\tau)$ is $\Lcal(\phi,\psi;\tau)$.
Assuming the encoder outputs variance $\sigma^2_i$ of each latent dimension, one choice for $\eta$ is a normalized standard deviation over all dimensions of the latent variable, i.e. $\eta := \frac{1}{D} \sum_{i=1}^D \sigma_i(q_{\phi})/\sigma_{i,\textrm{max}}(q_{\phi})$
, where $\sigma_{i,\textrm{max}}$ is a running max of $\sigma_i$.
% To standardize the scale of $R_p$ and $R_{env}$, one can scale-and-shift $R_p$ using the running mean and standard deviation of $\Lcal(\phi,\psi;\tau)$ and $R_{env}$.
% \begin{align}
%     R_p(\tau) := \Bigl( \frac{\Lcal(\phi,\psi;\tau) - \mu_{vae}}{\sigma_{vae}} + \mu_{env} \Bigl) \sigma_{env}
% \end{align}

Despite its novelty, we consider DynaSEPT only for rare nonstationary dynamics and merely as a baseline in the predominant case of stationary dynamics, where SEPT is our primary contribution.
% \Cref{app:dynasept} explains that DynaSEPT has no advantage over SEPT in the stationary case.
DynaSEPT does not have any clear advantage over SEPT when each instance $\Tcal_z$ is a stationary MDP.
DynaSEPT requires $\eta$ to start at 1.0, representing complete lack of knowledge about latent variables, and it still requires the choice of hyperparameter $T_p$.
Only after $T_p$ steps can it use the uncertainty of $q_{\phi}(z|\tau)$ to adapt $\eta$ and continue to generate the sliding window trajectory to improve $\hat{z}$.
By this time, SEPT has already generated an optimized sequence using $\pi_{\varphi}$ for the encoder to estimate $\zhat$.
If a trajectory of length $T_p$ is sufficient for computing a good estimate of latent variables, then SEPT is expected to outperform DynaSEPT.

\begin{figure}[H]
\centering
% \begin{subfigure}[t]{0.25\linewidth}
%     \centering
%     \includegraphics[width=1.0\linewidth]{figures/2D_ablation_steps_dyna.pdf}
%     \caption{2D navigation}
%     \label{fig:2D-ablation-dyna}
% \end{subfigure}
% \hfill
\begin{subfigure}[t]{0.32\linewidth}
    \centering
    \includegraphics[width=1.0\linewidth]{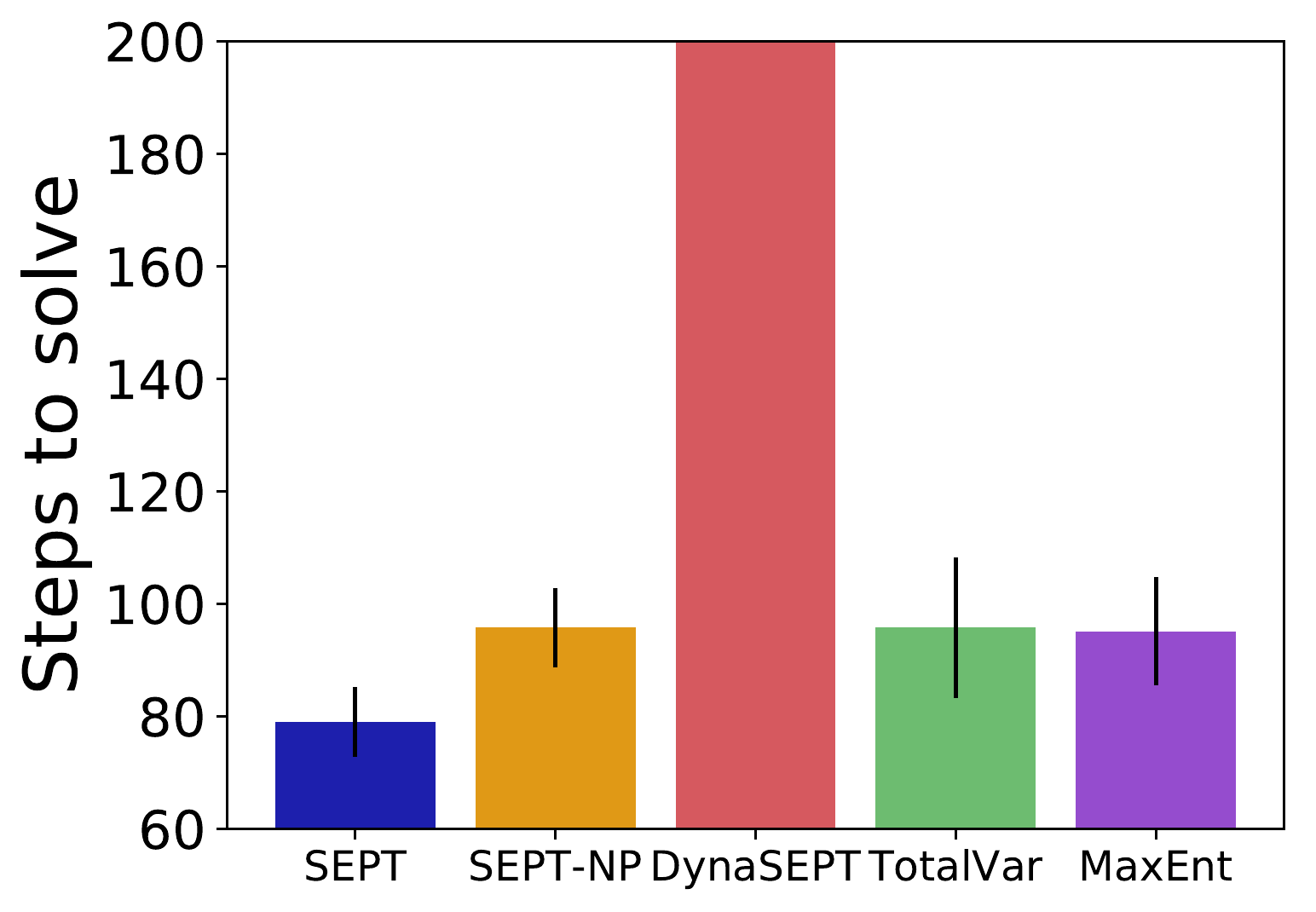}
    \caption{Acrobot}
    \label{fig:acrobot-ablation-dyna}
\end{subfigure}
\hfill
\begin{subfigure}[t]{0.32\linewidth}
    \centering
    \includegraphics[width=0.87\linewidth]{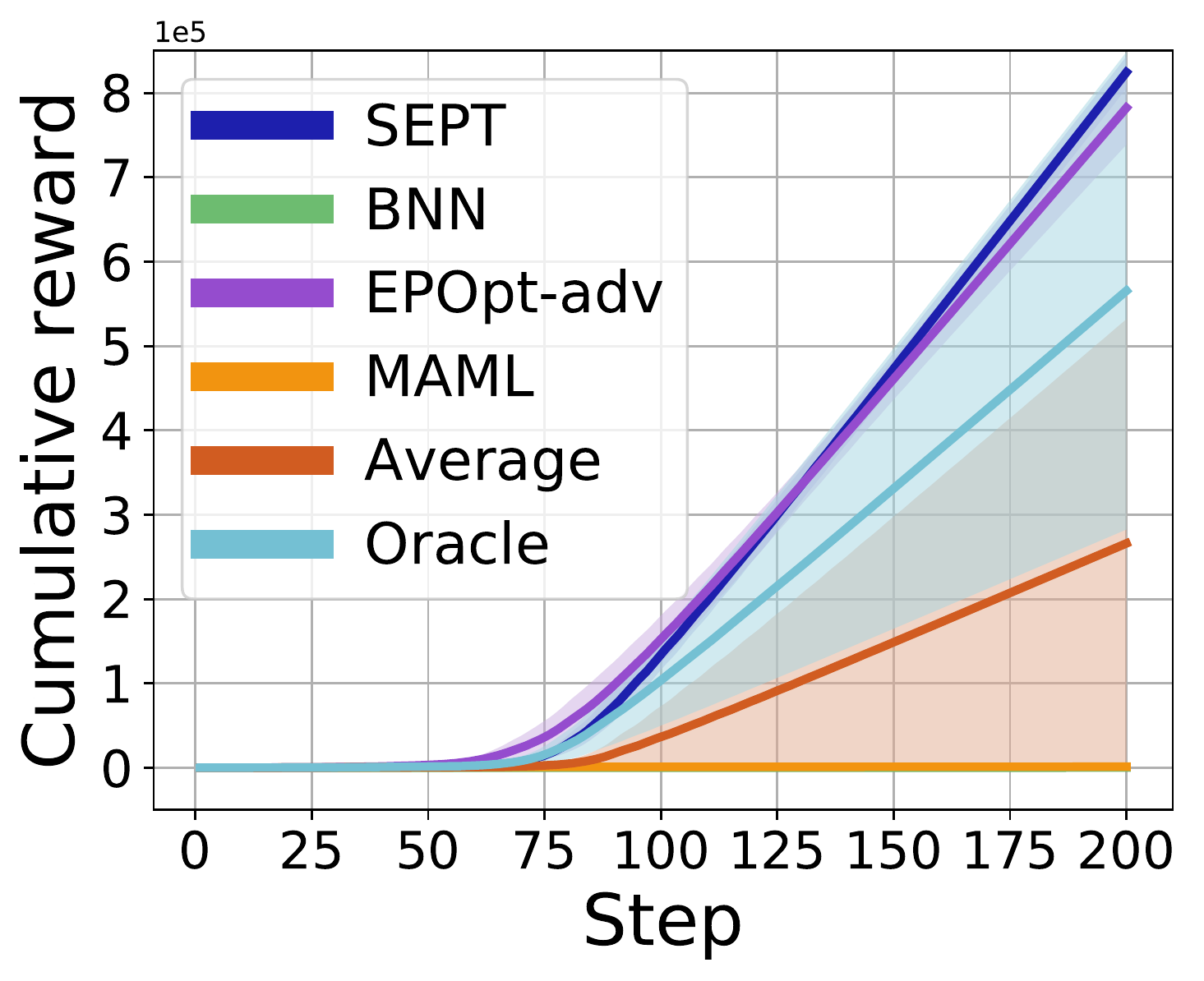}
    \caption{HIV}
    \label{fig:hiv-with-bnn}
\end{subfigure}
\hfill
\begin{subfigure}[t]{0.32\linewidth}
    \centering
    \includegraphics[width=0.9\linewidth]{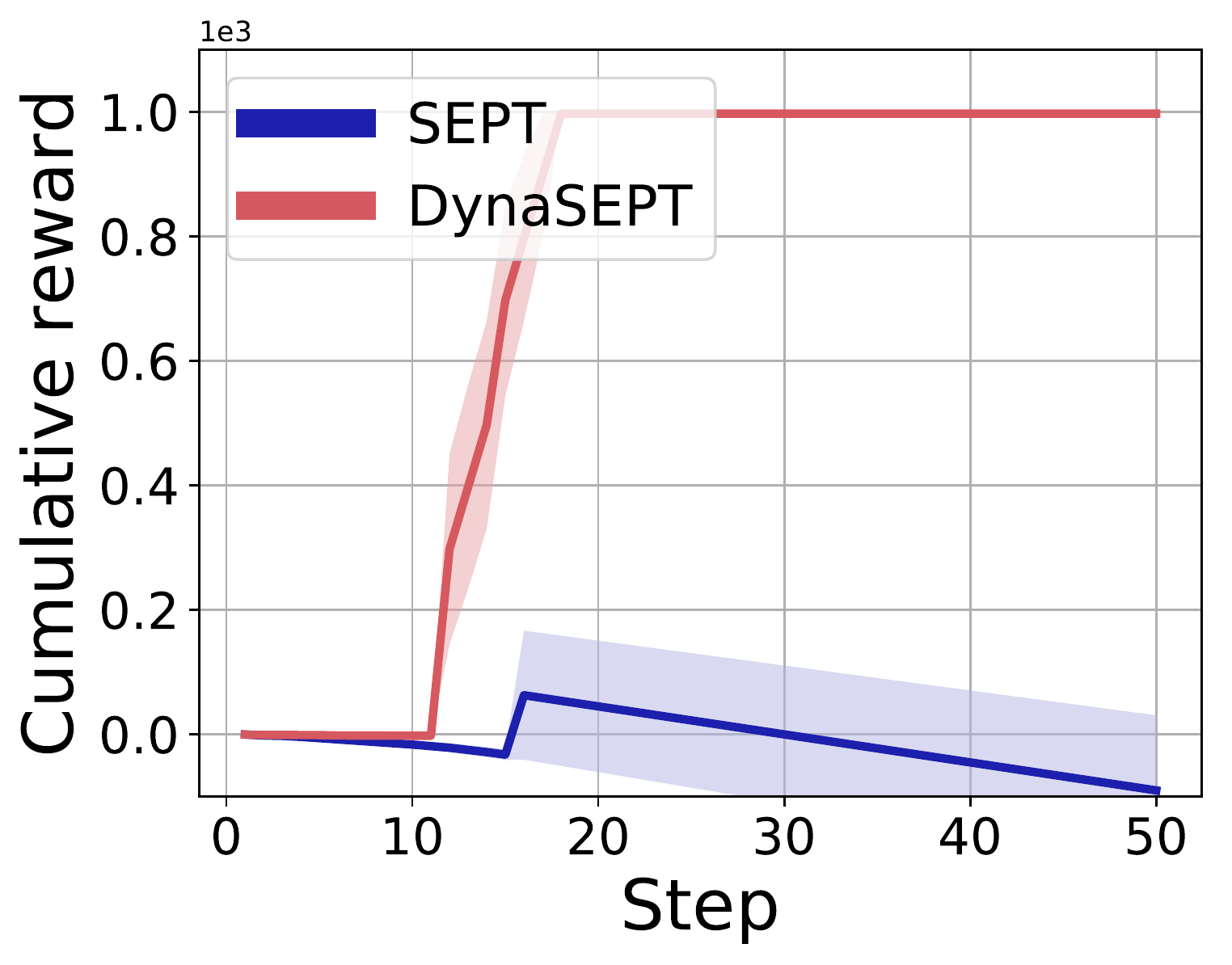}
    \caption{2D switch}
    \label{fig:2D-switch}
\end{subfigure}
\caption{(a) Ablations on Acrobot, including DynaSEPT with 3 seeds; (b) BNN and MAML attained orders of magnitude lower rewards than other baselines (3 seeds); (c) DynaSEPT performs well on nonstationary dynamics in 2D navigation.}
% \caption{Ablations and variants of SEPT (a-b); HIV including BNN (c); and additional test on nonstationary dynamics (d).}
\label{fig:ablations-dyna}
\end{figure}

\section{Supplementary experimental results}
\label{app:results}

\subsection{2D and Acrobot}
\label{app:solve-steps}

2D navigation and Acrobot have a definition of ``solved''.
\Cref{table:solve-steps} reports the number of steps in a test episode required to solve the MDP.
Average and standard deviation were computed across all test instances and across all independently trained models.
If an episode was not solved, the maximum allowed number of steps was used (50 for 2D navigation and 200 for Acrobot).
\Cref{table:acrobot-errorbar} shows the mean and standard error of the cumulative reward over test episodes on Acrobot.
The reported mean cumulative value for DPT in Acrobot is -27.7 \citep{yao2018direct} .

\begin{table}[H]
\parbox{.65\linewidth}{
  \caption{Steps to solve 2D navigation and Acrobot}
  \label{table:solve-steps}
  \centering
  \begin{tabular}{llll}
    \toprule
	 & 2D navigation & Acrobot \\
    \midrule
    % Average & 47$\pm$10 & 109$\pm$42 \\
    % Oracle & 12$\pm$1 & 82$\pm$27 \\
    % BNN & 34$\pm$10 & 154$\pm$45 \\
    % EPOpt-adv & 46$\pm$10 & 89$\pm$42 \\
    % MAML & 50$\pm$0 & 103$\pm$48 \\
    % SEPT & 14$\pm$2 & 79$\pm$23 \\
    Average & 49$\pm$0.5 & 114$\pm$4.2 \\
    Oracle & 12$\pm$0.3 & 88$\pm$3.6 \\
    BNN & 34$\pm$0.8 & 169$\pm$4.0 \\
    EPOpt-adv & 49$\pm$0.3 & 91$\pm$3.8 \\
    MAML & 49$\pm$0.4 & 99$\pm$4.6 \\
    SEPT & 20$\pm$0.9 & 100$\pm$4.4 \\
    \bottomrule
  \end{tabular}
  }
  \hfill
  \parbox{.32\linewidth}{
  \caption{Error bars on Acrobot}
  \label{table:acrobot-errorbar}
  \centering
  \begin{tabular}{llll}
    \toprule
	 & Acrobot \\
    \midrule
    Average & -22.2$\pm$2.3 \\
    Oracle & -17.1$\pm$2.2 \\
    BNN & -50.6$\pm$4.7 \\
    EPOpt-adv & -17.5$\pm$2.2 \\
    MAML & -20.1$\pm$2.6 \\
    SEPT & -23.1$\pm$3.1 \\
    \bottomrule
  \end{tabular}
  }
\end{table}

% \begin{table}[H]
%   \caption{Error bars on Acrobot}
%   \label{table:solve-steps}
%   \centering
%   \begin{tabular}{llll}
%     \toprule
% 	 & Acrobot \\
%     \midrule
%     Average & -22.2$\pm$2.3 \\
%     Oracle & -17.1$\pm$2.2 \\
%     BNN & -50.6$\pm$4.7 \\
%     EPOpt-adv & -17.5$\pm$2.2 \\
%     MAML & -20.1$\pm$2.6 \\
%     SEPT & -23.1$\pm$3.1 \\
%     \bottomrule
%   \end{tabular}
% \end{table}

\subsection{Timing comparison}
\label{app:timing}

\begin{table}[H]
  \caption{Total training times in seconds on all experiment domains}
  \label{table:train-times}
  \centering
  \begin{tabular}{lllll}
    \toprule
	 & 2D navigation & Acrobot & HIV \\
    \midrule
    Average & 1.3e3$\pm$277 & 1.0e3$\pm$85 & 1.4e3$\pm$47 \\
    Oracle & 0.6e3$\pm$163 & 1.1e3$\pm$129 & 1.5e3$\pm$47 \\
    BNN & 2.9e3$\pm$244 & 9.0e4$\pm$3.0e3 & 4.3e4$\pm$313 \\
    EPOpt-adv & 1.1e3$\pm$44 & 1.1e3$\pm$1.0 & 1.9e3$\pm$33 \\
    MAML & 0.9e3$\pm$116 & 1.1e3$\pm$96 & 1.3e3$\pm$6.0 \\
    SEPT & 1.9e3$\pm$70 & 2.3e3$\pm$1e3 & 2.8e3$\pm$11 \\
    % Cont & 8.0e3 $\pm$ 303 & 1.7e4 $\pm$ 778 & 8.7e3 $\pm$ 32 & & \\ 
    \bottomrule
  \end{tabular}
\end{table}

\begin{table}[H]
  \caption{Test episode time in seconds on all experiment domains}
  \label{table:test-times}
  \centering
  \begin{tabular}{lllll}
    \toprule
	 & 2D navigation & Acrobot & HIV \\
    \midrule
    Average & 0.04$\pm$0.04 & 0.09$\pm$0.04 & 0.42$\pm$0.01 \\
    Oracle & 0.02$\pm$0.04 & 0.09$\pm$0.04 & 0.45$\pm$0.02 \\
    BNN & 2.6e3$\pm$957 & 2.8e3$\pm$968 & 1.4e3$\pm$8.8 \\
    EPOpt-adv & 0.04$\pm$0.04 & 0.10$\pm$0.06 & 0.45$\pm$0.03 \\
    MAML & 0.05$\pm$0.05 & 0.10$\pm$0.07 & 0.48$\pm$0.01 \\
    SEPT & 0.04$\pm$0.07 & 0.12$\pm$0.10 & 0.60$\pm$0.02 \\
    % Cont & 0.15 $\pm$ 0.11 & 0.89 $\pm$ 0.27 & 1.6 $\pm$ 0.02 & & \\ 
    \bottomrule
  \end{tabular}
\end{table}

\subsection{Percent of solved episodes}

\begin{figure}[H]
\centering
\begin{subfigure}[t]{0.49\linewidth}
    \centering
    \includegraphics[width=0.7\linewidth]{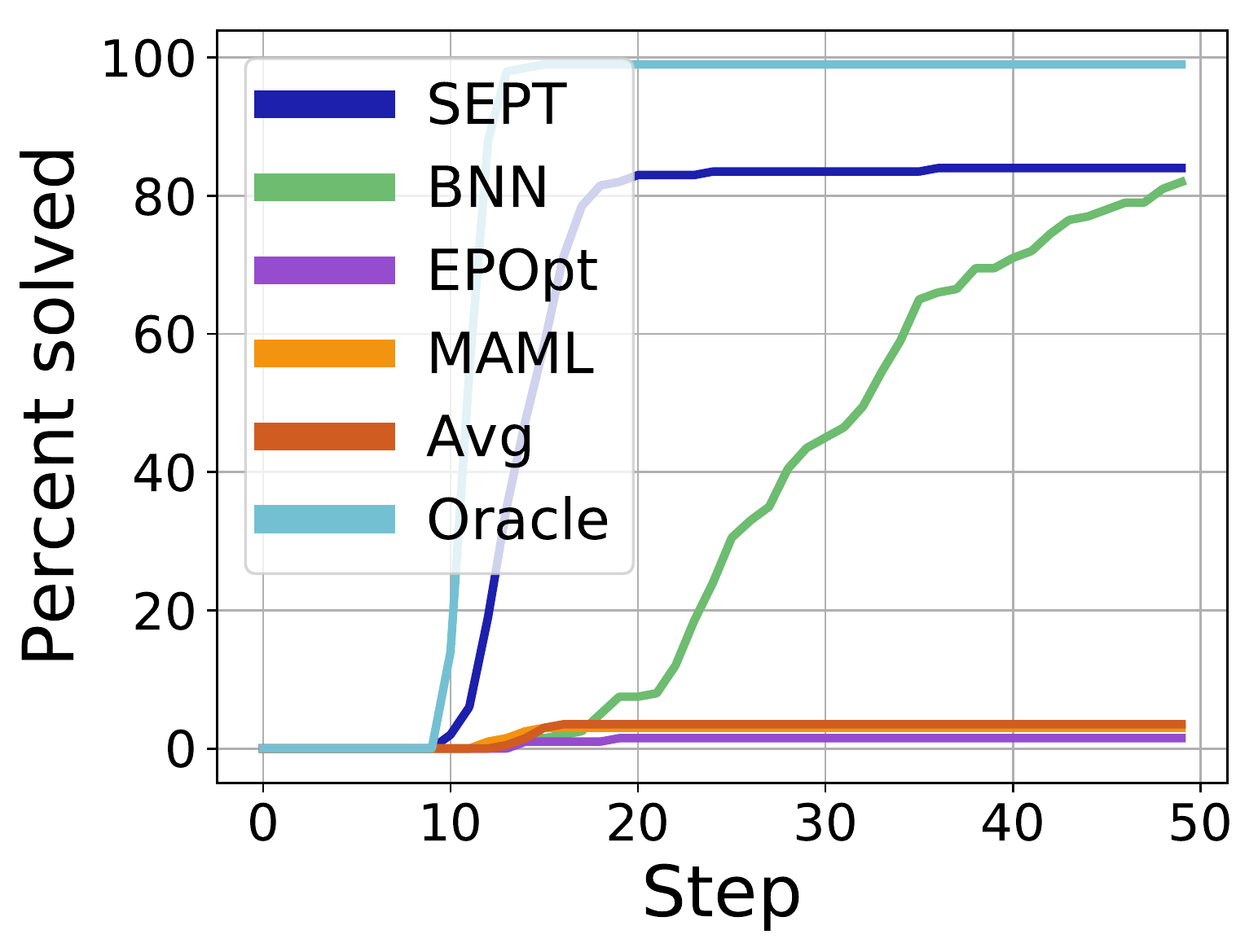}
    \caption{2D navigation}
    \label{fig:2D-percent}
\end{subfigure}
\hfill
\begin{subfigure}[t]{0.49\linewidth}
    \centering
    \includegraphics[width=0.7\linewidth]{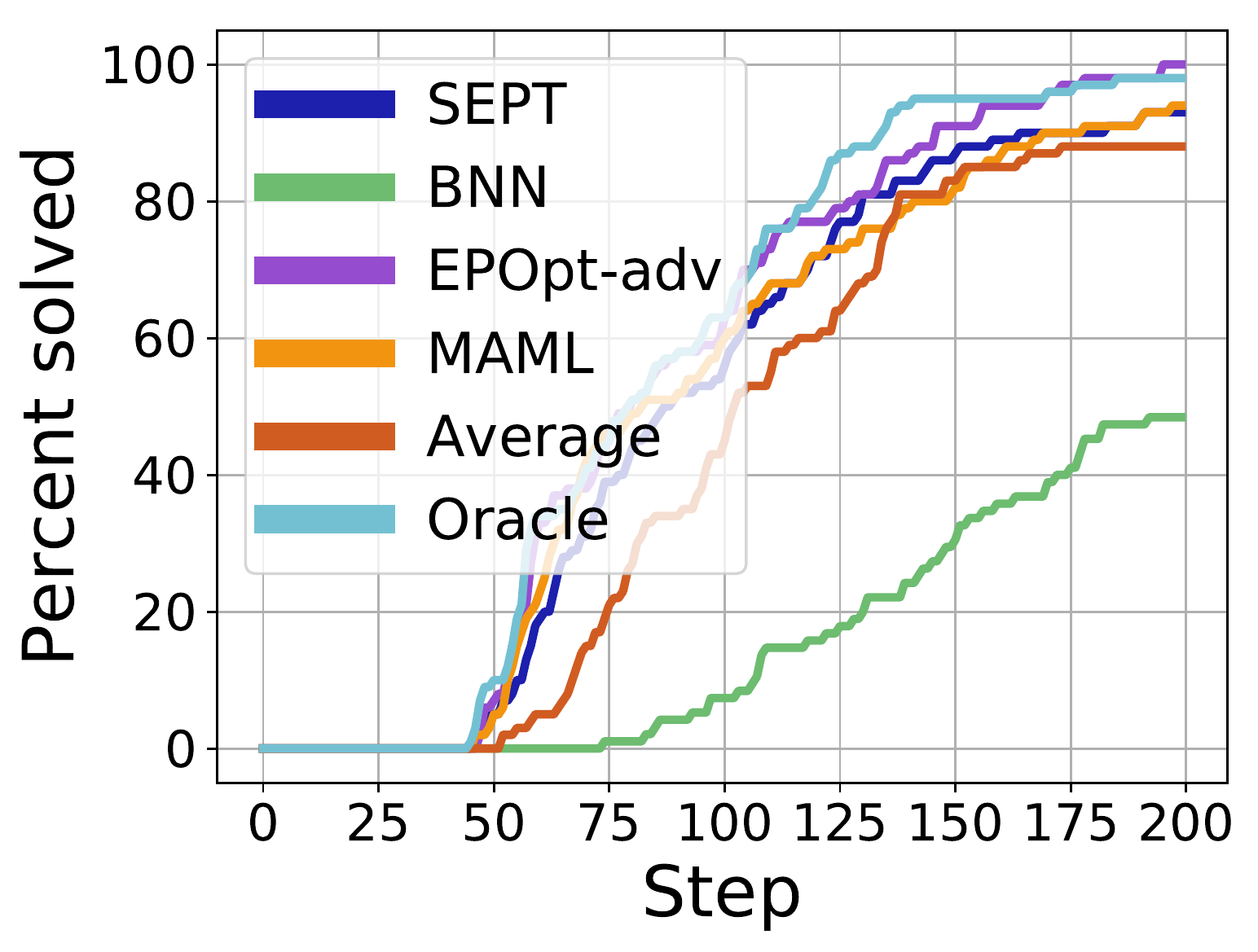}
    \caption{Acrobot}
    \label{fig:acrobot-percent}
\end{subfigure}
\caption{Percent of solved test instances as a function of time steps during the test episode. Percentage is computed among 200 test instances for (a) 2D navigation and (b) 100 test instances for Acrobot.}
\label{fig:percent}
\end{figure}

2D navigation and Acrobot are considered solved upon reaching a terminal reward of 1000 and 10, respectively.
Figure \Cref{fig:percent} shows the percentage of all test episodes that are solved as a function of time steps in the test episode.
Percentage is measured from a total of 200 test episodes for 2D navigation and 100 test episodes for Acrobot.

\subsection{Training curves}

\begin{figure}[H]
\centering
\begin{subfigure}[t]{0.30\linewidth}
    \centering
    \includegraphics[width=1.0\linewidth]{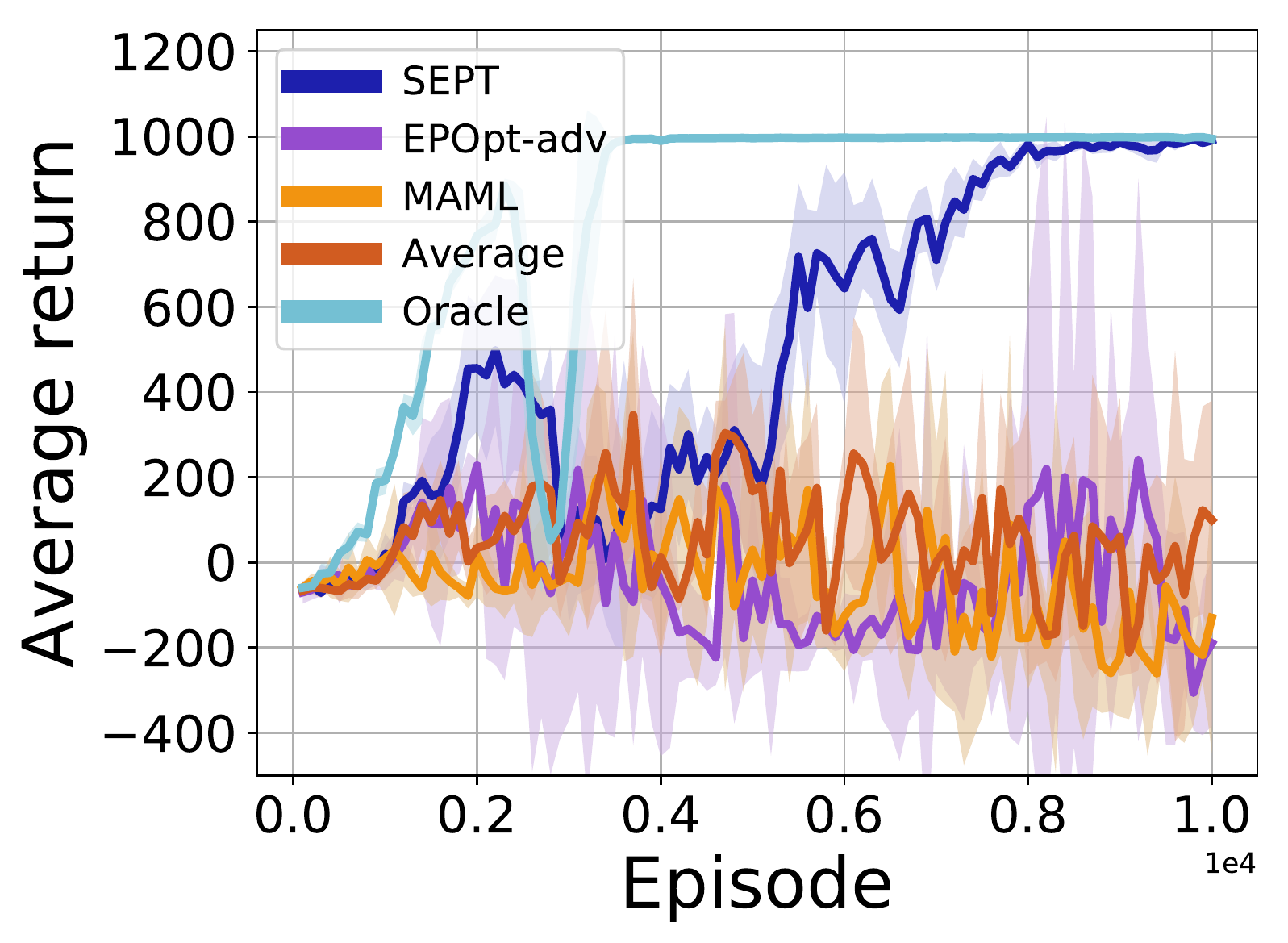}
    \caption{2D navigation}
    \label{fig:2D-training}
\end{subfigure}
\hfill
\begin{subfigure}[t]{0.30\linewidth}
    \centering
    \includegraphics[width=1.0\linewidth]{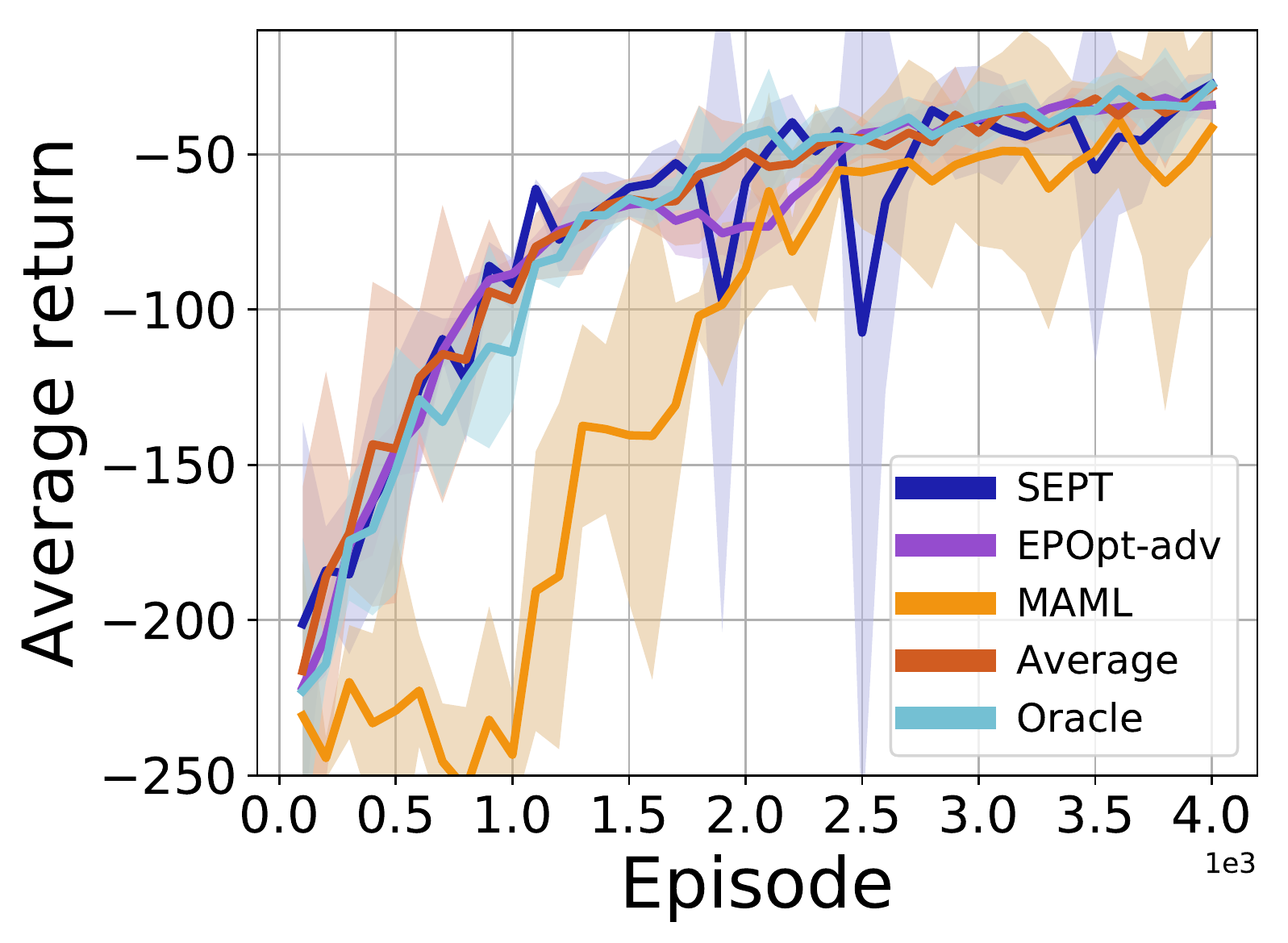}
    \caption{Acrobot}
    \label{fig:acrobot-training}
\end{subfigure}
\hfill
\begin{subfigure}[t]{0.30\linewidth}
    \centering
    \includegraphics[width=0.9\linewidth]{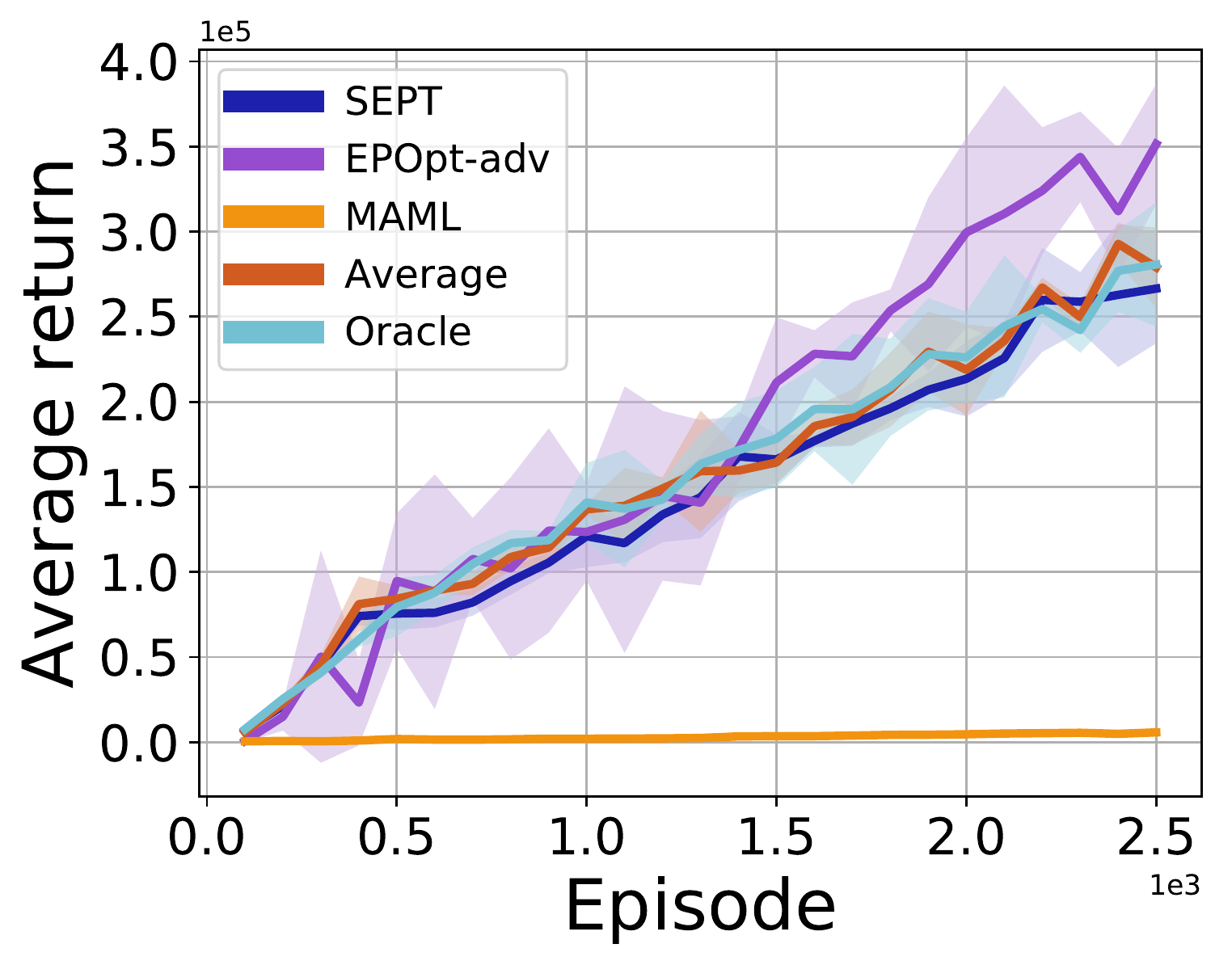}
    \caption{HIV}
    \label{fig:hiv-training}
\end{subfigure}
\caption{Average episodic return over training episodes. Only SEPT and Oracle converged in 2D navigation. All methods converged in Acrobot. All methods except MAML converged in HIV. BNN is not shown as the implementation \citep{killian2017robust} does not record training progress.}
\label{fig:training}
\end{figure}

\Cref{fig:training} shows training curves on all domains by all methods.
None of the baselines, except for Oracle, converge in 2D navigation, because it is meaningless for Avg and EPOpt-adv to interpolate between optimal policies for each instance, and MAML cannot adapt due to lack of informative rewards for almost the entire test episode.
Hence these baselines cannot work for a new unknown test episode, even in principle.
We allowed the same number of training episodes for HIV as in \citet{killian2017robust}, and all baselines except MAML show learning progress.

\subsection{Latent representation of dynamics}
\label{app:latent-dynamics}

\begin{figure}[H]
    \centering
    \includegraphics[width=0.5\linewidth]{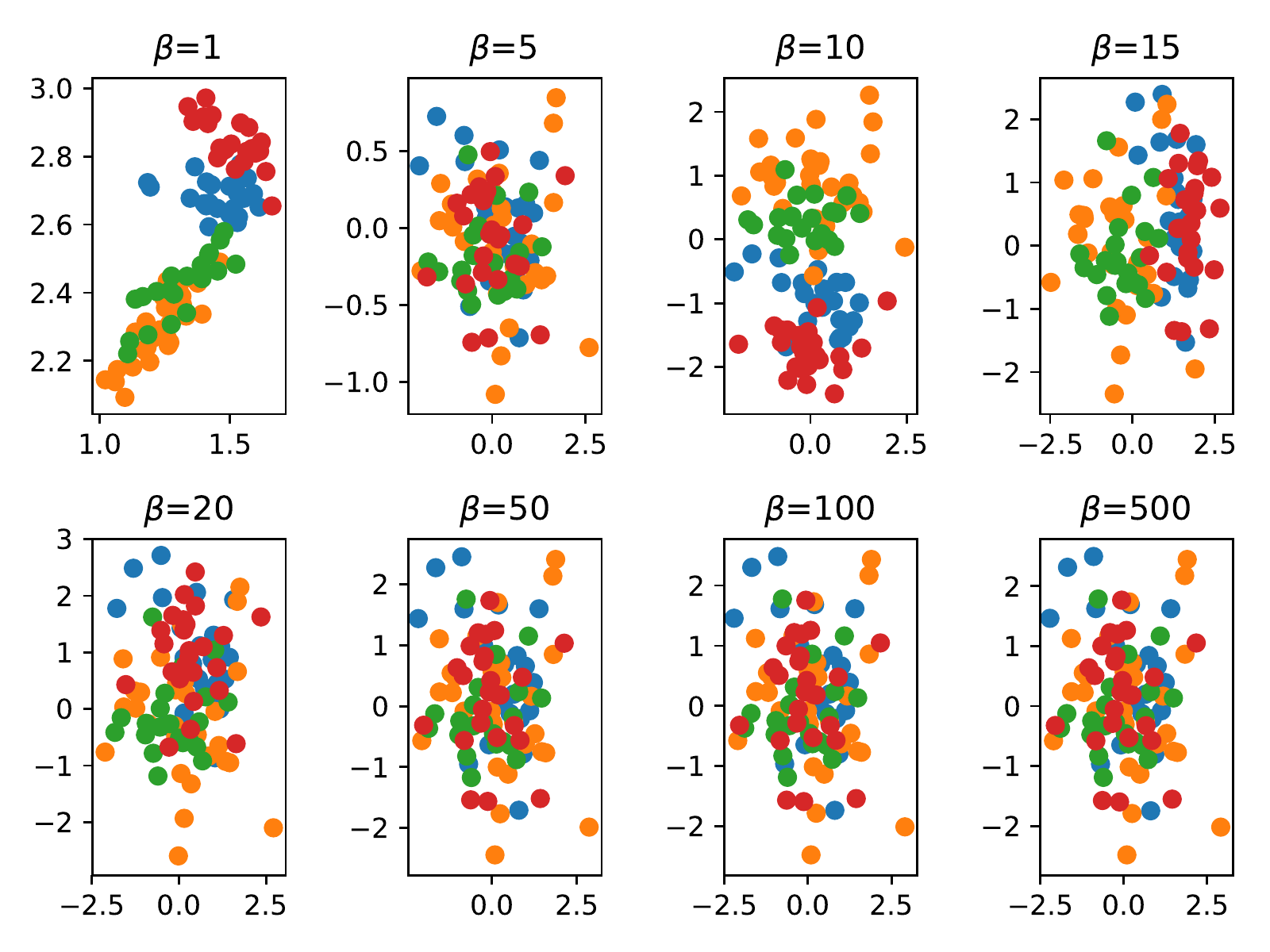}
    \caption{Two-dimensional encodings generated for four instances of Acrobot (represented by four ground-truth colors), for different values of $\beta$. We chose $\beta=1$ for Acrobot.}
    \label{fig:acrobot-latent}
\end{figure}

There is a tradeoff between reconstruction and disentanglement as $\beta$ increases \citep{higgins2017beta}.
Increasing $\beta$ encourages greater similarity between the posterior and an isotropic Gaussian.
\Cref{fig:acrobot-latent} gives evidence that this comes at a cost of lower quality of separation in latent space.

\subsection{Probe reward}
\label{app:probe-reward}

\begin{figure}[H]
    \centering
    \includegraphics[width=0.25\linewidth]{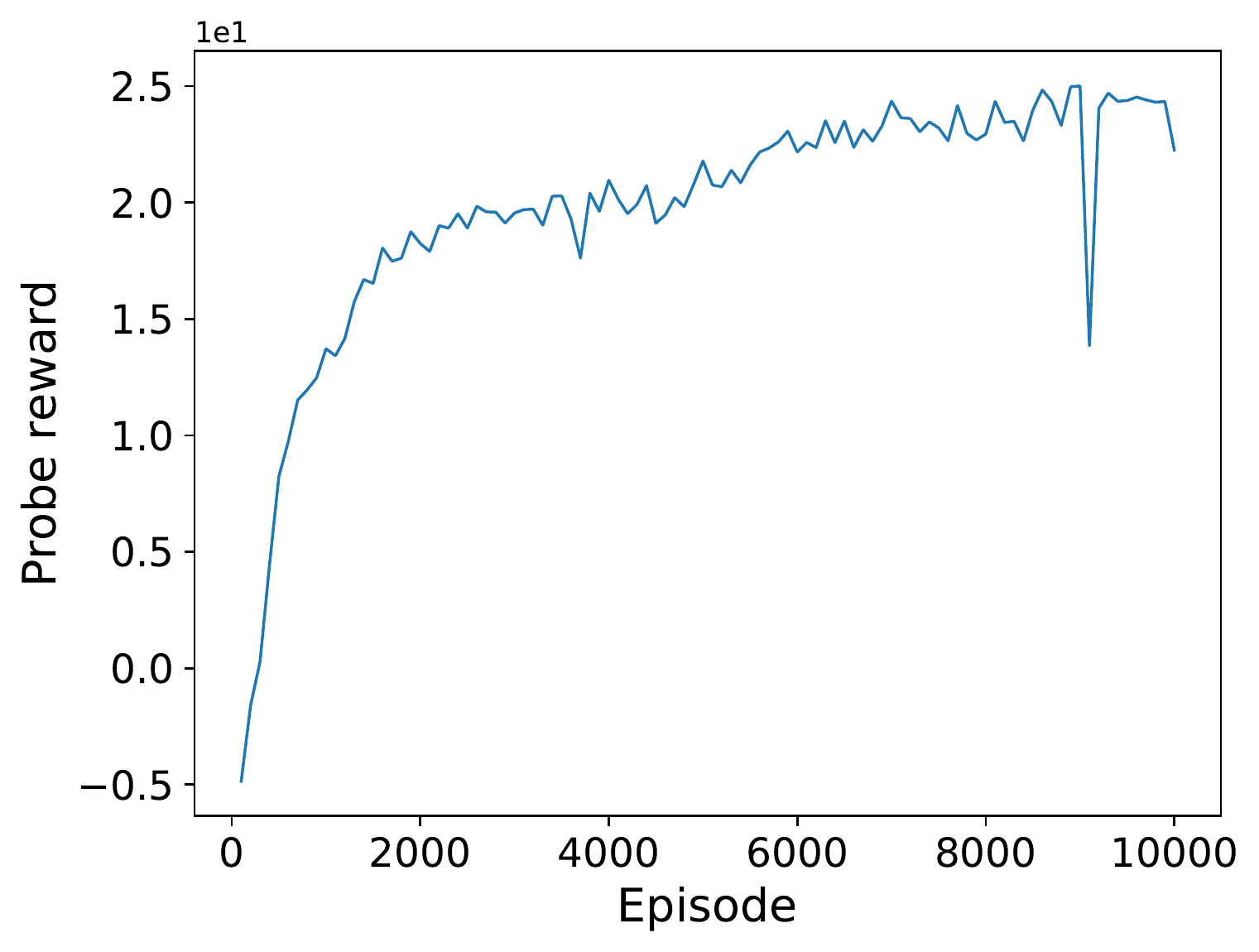}
    \caption{Probe policy reward curve in one training run in 2D navigation}
    \label{fig:2D-probe}
\end{figure}

\section{Experimental details}
\label{app:experiment}

% \subsection{Domains}
% \label{app:domains}

For 2D navigation, Acrobot, and HIV, total number of training episodes allowed for all methods are 10k, 4k, and 2.5k, respectively.
We switch instances once every 10, 8 and 5 episodes, respectively.
There are 2, 8 and 5 unique training instances, and 2, 5, and 5 validation instances, respectively.
For each independent training run, we tested on 10, 5, and 5 test instances, respectively.

% \textbf{2D flipped world.}
% This domain only has two instances.
% We iterated over 1000 randomly sampled instances, for 10 episodes per instance.

% \textbf{Acrobot.}
% \citet{killian2017robust} used 8 instances with 500 episodes per instance.
% In our algorithm, we iterate over 500 randomly sampled instances (from a set of 8 unique instances) for 8 episodes per instance.

% \textbf{HIV.}
% \citet{killian2017robust} used 5 instances with 500 episodes per instance.
% In our experiment, we iterate over 500 randomly sampled instances (from that set of 5 unique instances) for 5 episodes per instance.

\subsection{Algorithm implementation details}
\label{app:algorithms}

The simple baselines Average and Oracle can be immediately deployed in a single test episode after training.
However, the other methods for transfer learning require modification to work in the setting of single episode test, as they were not designed specifically for this highly constrained setting.
We detail the necessary modifications below.
We also describe the ablation SEPT-NP in more detail.

% \paragraph{BNN}
\textbf{BNN.}
In \citet{killian2017robust}, a pre-trained BNN model was fine-tuned using the first test episode and then used to generate fictional episodes for training a policy from scratch.
More episodes on the same test instance were allowed to help improve model accuracy of the BNN.
In the single test episode setting, all fine-tuning and policy training must be conducted within the first test episode.
We fine-tune the pre-trained BNN every 10 steps and allow the same total number of fictional episodes as reported in \citep{killian2017robust} for policy training.
We measured the cumulative reward attained by the policy---while it is undergoing training---during the single real test episode.

% \paragraph{EPOpt}
\textbf{EPOpt.}
EPOpt trains on the lowest $\epsilon$-percentile rollouts from instances sampled from a source distribution, then adapts the source distribution using observations from the target instance \citep{rajeswaran2016epopt}.
Since we do not allow observation from the test instance, we only implemented the adversarial part of EPOpt.
To run EPOpt with off-policy DDQN, we generated 100 rollouts per iteration and stored the lowest 10-percentile into the replay buffer, then executed the same number of minibatch training steps as the number that a regular DDQN would have done during rollouts.

% \paragraph{MAML}
\textbf{MAML.}
While MAML uses many complete rollouts per gradient step \citep{finn2017model}, the single episode test constraint mandates that it can only use a partial episode for adaptation during test, and hence the same must be done during meta-training.
For both training and test, we allow MAML to take one gradient step for adaptation using a trajectory of the same length as the probe trajectory of SEPT, starting from the initial state of the episode.
We implemented a first-order approximation that computes the meta-gradient at the post-update parameters but omits second derivatives.
This was reported to have nearly equal performance as the full version, due to the use of ReLU activations.

% \paragraph{SEPT-NP}
\textbf{SEPT-NP.}
$\pi_{\theta}(a|s,z)$ begins with a zero-vector for $z$ at the start of training.
When it has produced a trajectory $\tau_p$ of length $T_p$, we store $\tau_p$ into $\Dcal$ for training the VAE, and use $\tau_p$ with the VAE to estimate $z$ for the episode.
Later training episodes begin with the rolling mean of all $z$ estimated so far.
For test, we give the final rolling mean of $z$ at the end of training as initial input to $\pi_{\theta}(a|s,z)$.

\subsection{Architecture}
\label{app:architecture}

% \paragraph{Encoder}
\textbf{Encoder.}
For all experiments, the encoder $q_{\phi}(z|\tau)$ is a bidirectional LSTM with 300 hidden units and $\tanh$ activation.
Outputs are mean-pooled over time, then fully-connected to two linear output layers of width dim($z$), interpreted as the mean and log-variance of a Gaussian over $z$.

% \paragraph{Decoder}
\textbf{Decoder.}
For all experiments, the decoder $p_{\psi}(\tau|z)$ is an LSTM with 256 hidden units and $\tanh$ activation.
Given input $[s_t,a_t,\zhat]$ at LSTM time step $t$, the output is fully-connected to two linear output layers of width $|\Scal| + |\Acal|$, and interpreted as the mean and log-variance of a Gaussian decoder for the next state-action pair $(s_{t+1},a_{t+1})$.

% \paragraph{Q network}
\textbf{Q network.}
For all experiments, the $Q$ function is a fully-connected neural network with two hidden layers of width 256 and 512, ReLU activation, and a linear output layer of size $|\Acal|$.
For SEPT and Oracle, the input is the concatenation $[s_t, z]$, where $z$ is estimated in the case of SEPT and $z$ is the ground truth in for the Oracle.
For all other methods, the input is only the state $s$.

% \paragraph{Probe policy network}
\textbf{Probe policy network.}
For all experiments, $\pi_{\varphi}(a|s)$ is a fully-connected neural network with 3 hidden layers, ReLU activation, 32 nodes in all layers, and a softmax in the output layer.

\subsection{Hyperparameters}

VAE learning rate was 1e-4 for all experiments.
Size of the dataset $\Dcal$ of probe trajectories was limited to 1000, with earliest trajectories discarded.
10 minibatches from $\Dcal$ were used for each VAE training step.
We used $\beta=1$ for the VAE.
Probe policy learning rate was 1e-3 for all experiments.
DDQN minibatch size was 32, one training step was done for every 10 environment steps, $\epsilon_{\text{end}} = 0.15$, learning rate was 1e-3, gradient clip was 2.5, $\gamma = 0.99$, and target network update rate was 5e-3.
Exploration decayed according $\epsilon_{n+1} = c \epsilon_{n}$ every episode, where $c$ satisfies $\epsilon_{\text{end}} = c^N \epsilon_{\text{start}}$ and $N$ is the total number of episodes.
Prioritized replay used the same parameters in \citep{killian2017robust}.

\begin{table}[h]
  \caption{Hyperparameters used by each method, where applicable}
  \label{table:hyperparameters}
  \centering
  \begin{tabular}{lllll}
    \toprule
	 & 2D navigation & Acrobot & HIV \\
    \midrule
    $T_p$ & 2 & 5 & 8 \\
    Instances & 1000 & 500 & 500 \\
    Episodes per instance & 10 & 8 & 5 \\
    VAE batch size & 10 & 64 & 64 \\
    dim($\zhat$) & 2 & 2 & 6 \\
    $\alpha$ & 1.0 & 0.005 & 1.0 \\
    Probe minibatches & 1 & 10 & 1 \\
    DDQN $\epsilon_{\text{start}}$ & 1.0 & 1.0 & 0.3 \\
    \bottomrule
  \end{tabular}
\end{table}

\end{document}